\documentclass{article} %

\usepackage{amsmath,amsfonts,amsthm,bm, bbm}

\newtheorem{proposition}{Proposition}
\newtheorem{corollary}{Corollary}

\def\eqref#1{equation~\ref{#1}}

\def\1{\bm{1}}

\DeclareMathAlphabet{\mathsfit}{\encodingdefault}{\sfdefault}{m}{sl}
\SetMathAlphabet{\mathsfit}{bold}{\encodingdefault}{\sfdefault}{bx}{n}

\newcommand{\E}{\mathop{\mathbb{E}}}

\DeclareMathOperator*{\argmax}{arg\,max}

\DeclareMathOperator{\Tr}{Tr}

\usepackage[preprint]{neurips_2024}

\newcommand\animage[1]{\adjustbox{valign=m,vspace=1pt}{\includegraphics[width=.29\linewidth]{#1}}}
\usepackage{graphicx}
\usepackage{adjustbox}
\usepackage{caption}
\usepackage{subcaption}

\usepackage{amsmath}
\usepackage{algorithm}
\usepackage{algpseudocode}
\usepackage{amsfonts}
\usepackage{graphicx}
\usepackage{thm-restate}
\usepackage{multirow}

\usepackage{algorithm}

\RequirePackage{fancyhdr}
\RequirePackage{natbib}
\usepackage{amsmath}
\usepackage{mathabx}

\usepackage{makecell}

\usepackage{hyperref}
\usepackage{url}
\usepackage{amsthm}
\usepackage{enumitem}
\usepackage{subcaption}
\usepackage{float} 
\usepackage{graphicx}
\usepackage{minitoc}
\usepackage{xcolor}
\usepackage{comment}
\usepackage{booktabs}
\usepackage{makecell}
\usepackage{mathtools}

\newcommand{\vect}[1]{\operatorname{vec}(#1)}
\newcommand{\vectb}[1]{\operatorname{vec}\left(#1\right)}

\newcommand{\reshape}[1]{\operatorname{reshape}(#1)}

\title{A New Perspective on Shampoo's Preconditioner}

\begin{document}

\author{Depen Morwani\footnotemark[1] \\ SEAS \\ Harvard University \\ \texttt{dmorwani@g.harvard.edu} \And Itai Shapira\footnotemark[1] \\ SEAS \\ Harvard University \\ \texttt{itaishapira@g.harvard.edu} \And Nikhil Vyas\thanks{Equal contribution. Randomized Author Ordering.} \\ SEAS \\ Harvard University \\ \texttt{nikhil@g.harvard.edu} \AND Eran Malach \\ Kempner Institute \\ Harvard University \\ \texttt{emalach@g.harvard.edu} \And Sham Kakade \\ Kempner Institute \\ Harvard University \\ \texttt{sham@seas.harvard.edu} \And Lucas Janson \\ Department of Statistics \\ Harvard University \\ \texttt{ljanson@g.harvard.edu} 
}

\maketitle

\begin{abstract}

Shampoo, a second-order optimization algorithm which uses a Kronecker product preconditioner, has recently garnered increasing attention from the machine learning community. The preconditioner used by Shampoo can be viewed either as an approximation of the Gauss--Newton component of the Hessian or the covariance matrix of the gradients maintained by Adagrad.
We provide an explicit and novel connection between the \emph{optimal} Kronecker product approximation of these matrices and the approximation made by Shampoo. Our connection highlights a subtle but common misconception about Shampoo's approximation. In particular, the \textit{square} of the approximation used by the Shampoo optimizer is equivalent to a single step of the power iteration algorithm for computing the aforementioned optimal Kronecker product approximation. Across a variety of datasets and architectures we empirically demonstrate that this is close to the optimal Kronecker product approximation. 
Additionally, for the Hessian approximation viewpoint, we empirically study the impact of various practical tricks to make Shampoo more computationally efficient (such as using the batch gradient and the empirical Fisher) on the quality of Hessian approximation.

\end{abstract}

\newtheorem{theorem}{Theorem}
\newtheorem*{theorem*}{Theorem}
\newtheorem{lemma}[theorem]{Lemma}
\newtheorem*{lemma*}{Lemma}
\newtheorem*{remark*}{Remark}

\section{Introduction}

\begin{figure}[htbp]
\centering
\begin{tabular}{cccc}
       & \textnormal{\small {MNIST-2}} & \textnormal{\small {CIFAR-5M}} & \textnormal{\small {ImageNet}} \\
\rotatebox[origin=c]{90}{\textnormal{\small {Gauss--Newton}}} & 
\animage{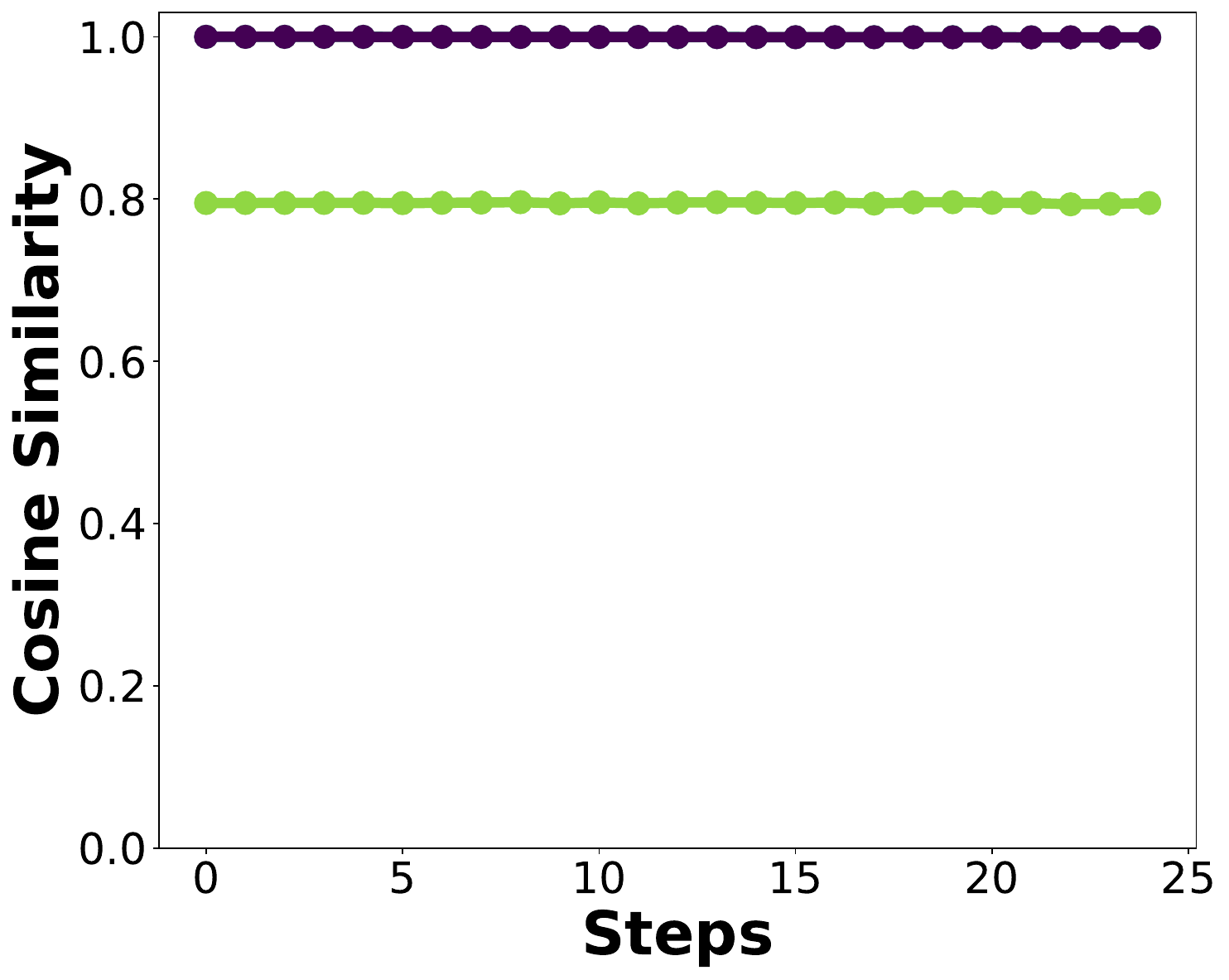} & 
\animage{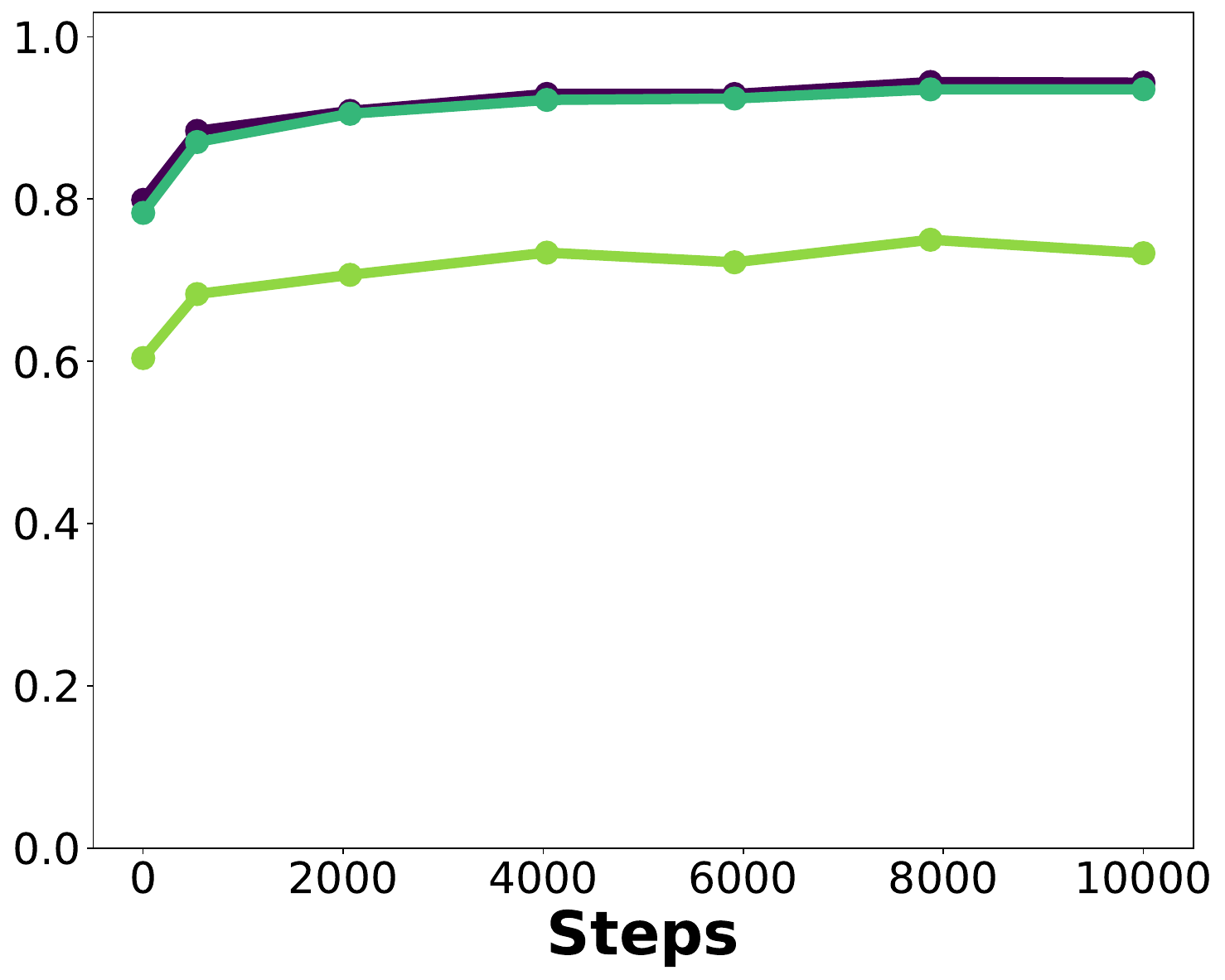} & 
\animage{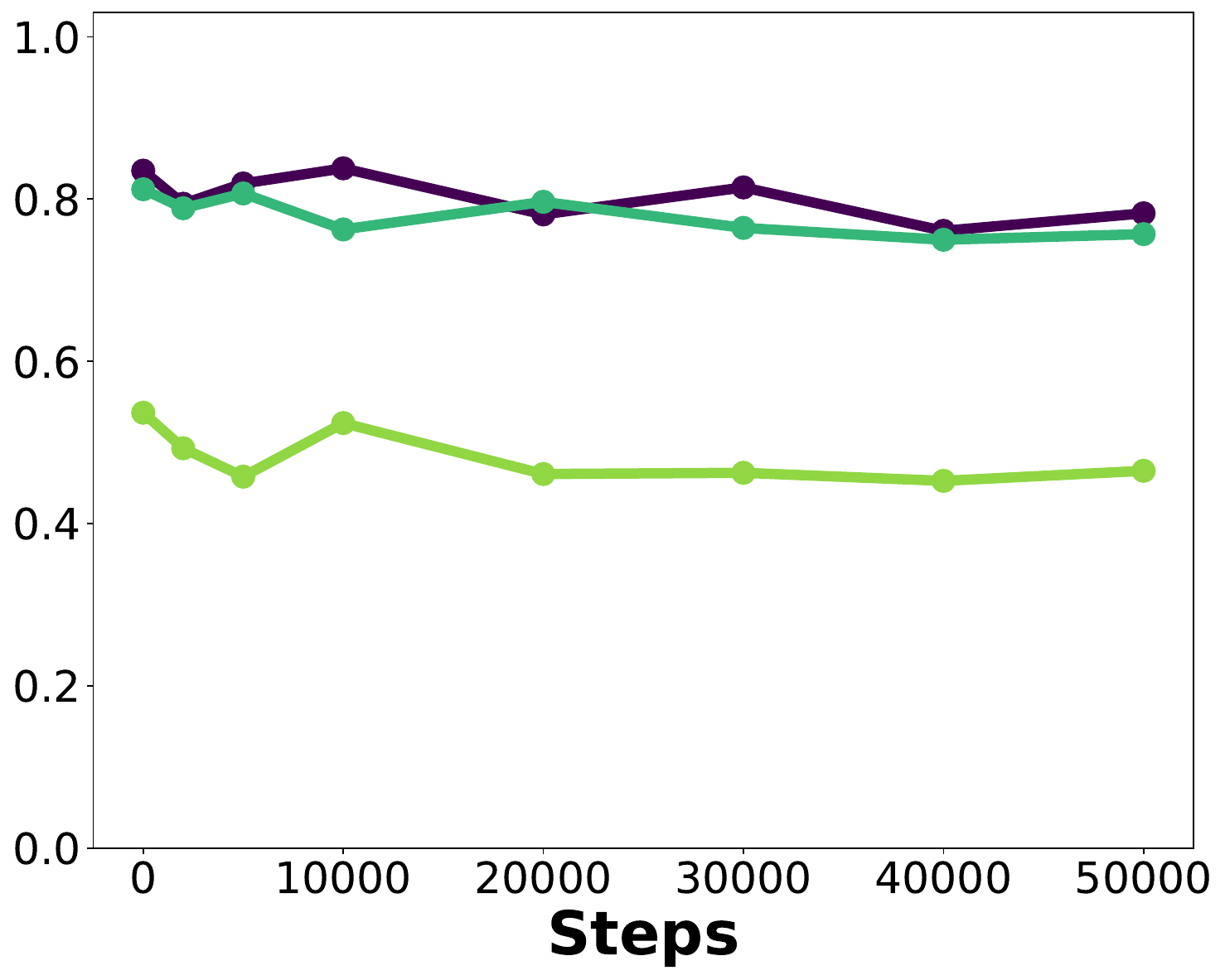} \\
\rotatebox[origin=c]{90}{\textnormal{\small {Adagrad}}} & 
\animage{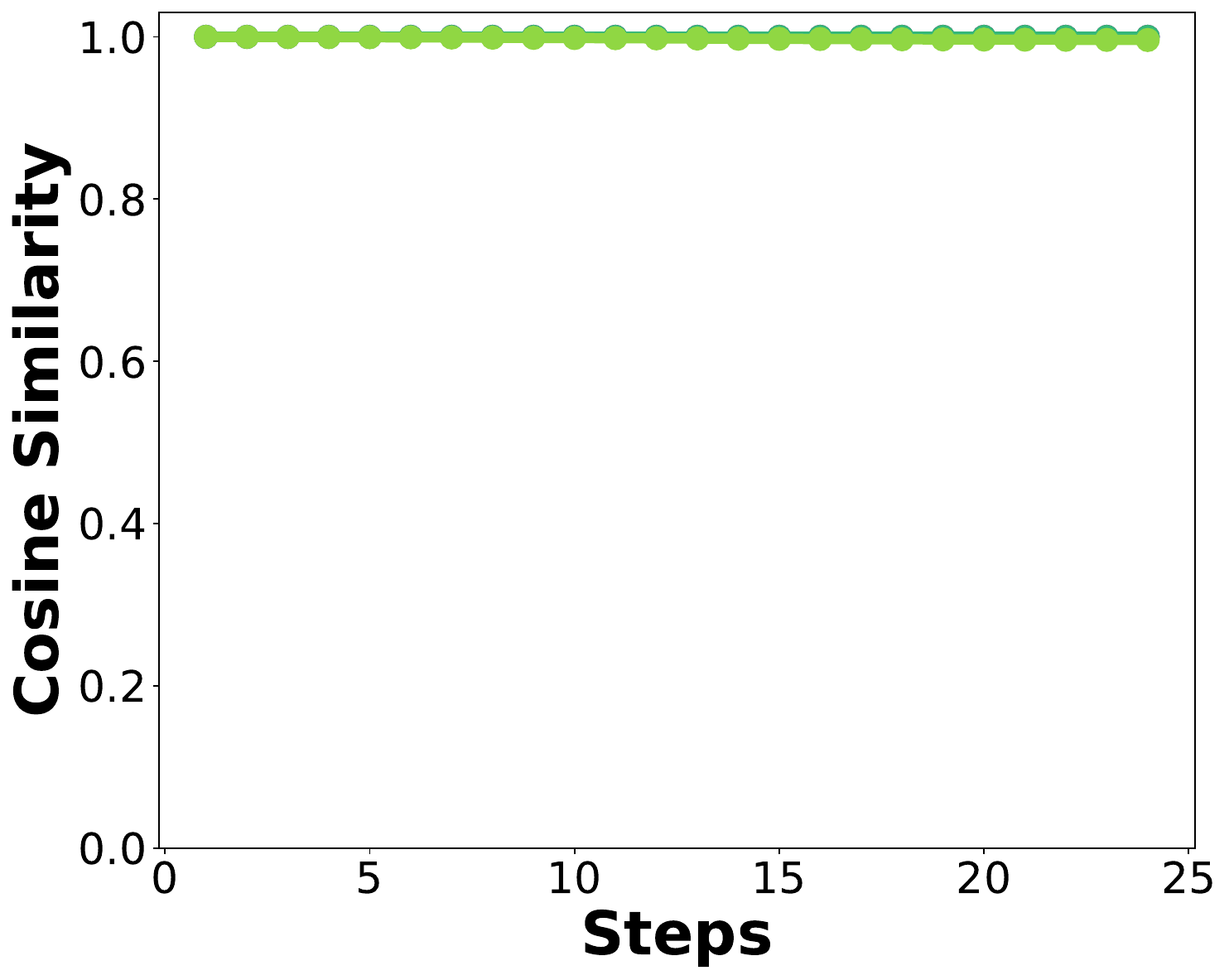} & 
\animage{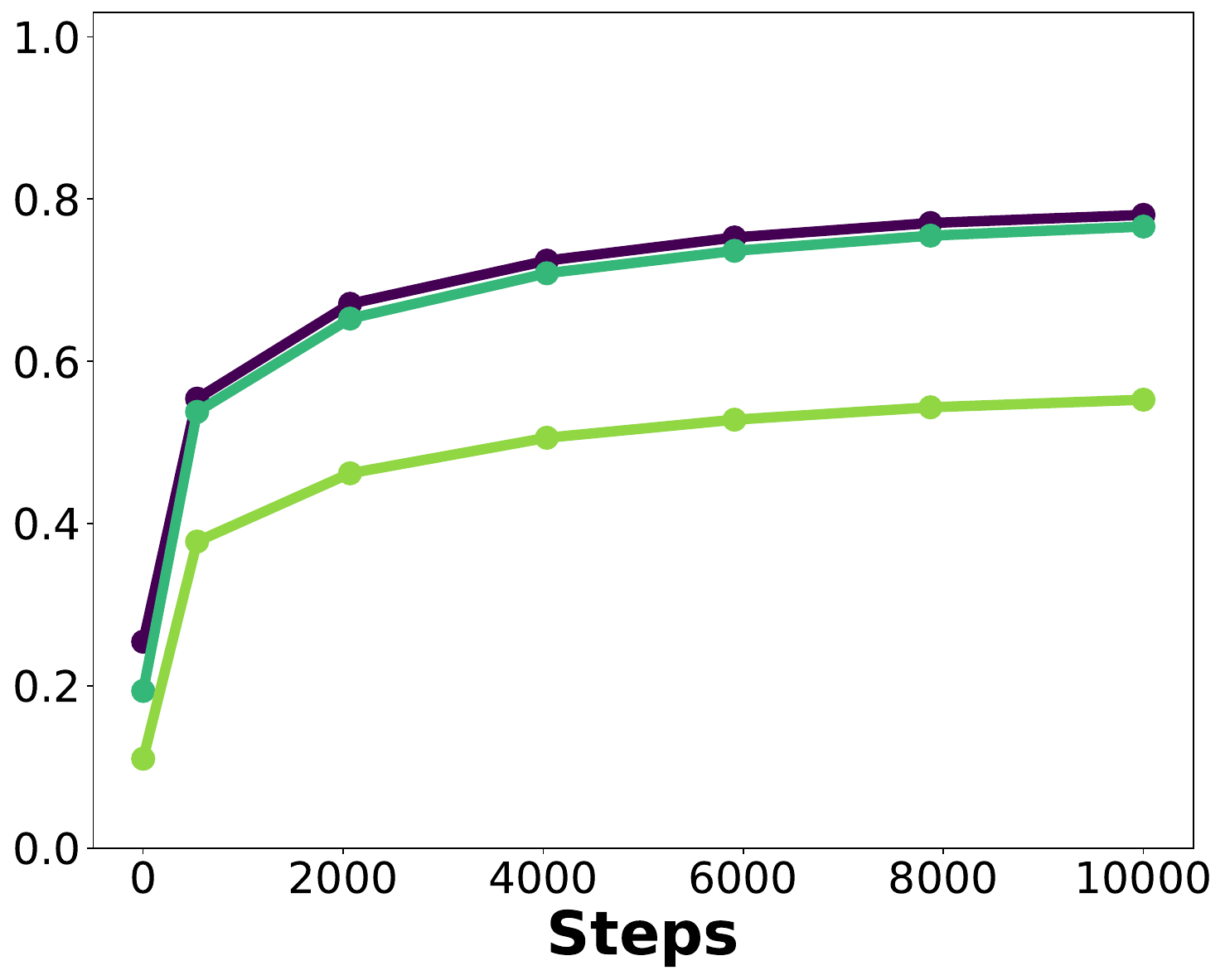} & 
\animage{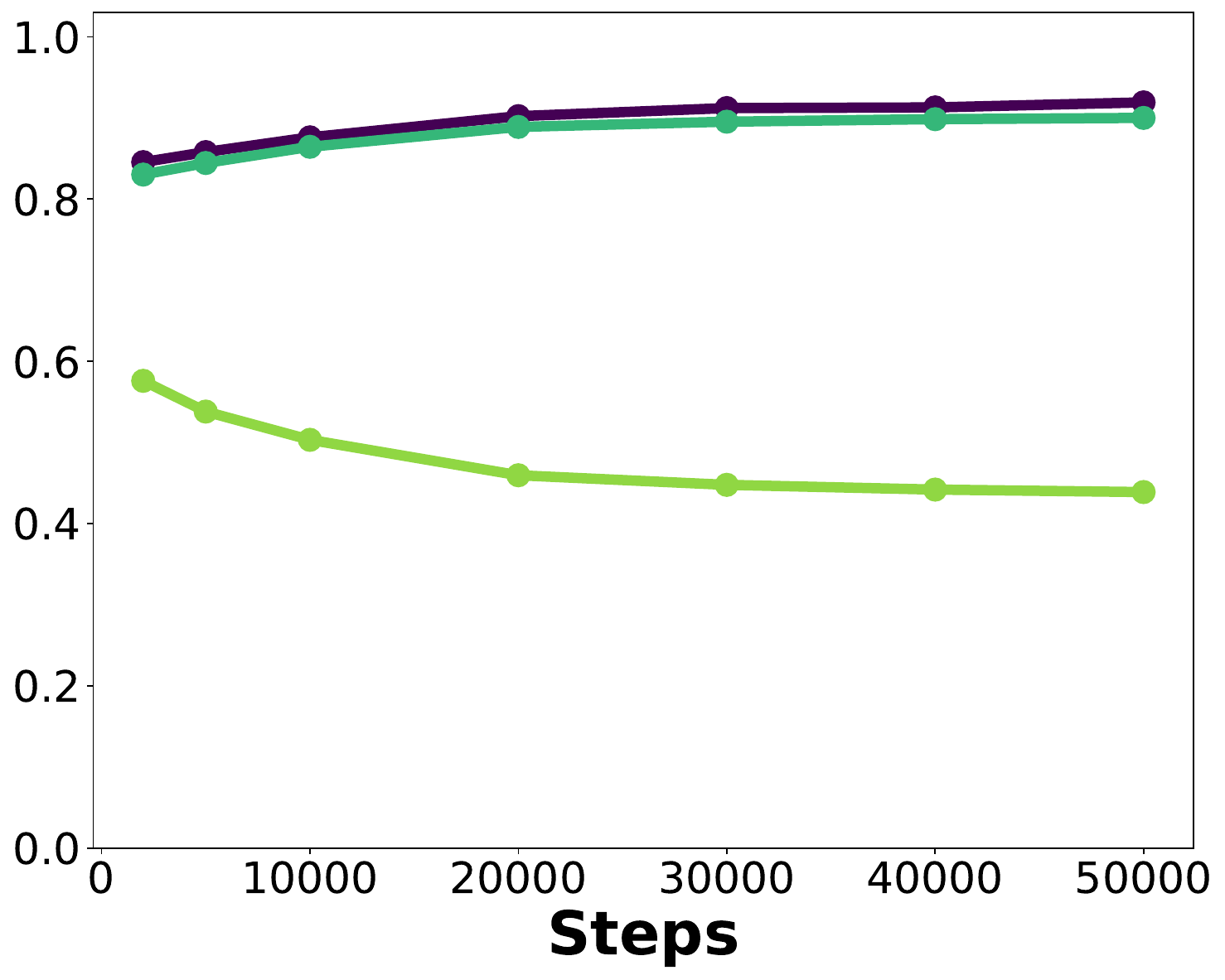} \\
\end{tabular}

\begin{minipage}{\textwidth}
    \centering
    \includegraphics[width=0.95\textwidth]{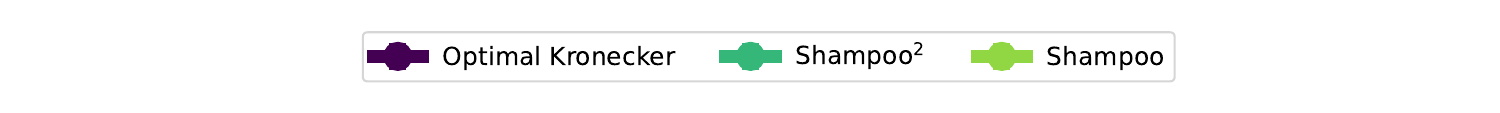}
\end{minipage}

\caption{Top: Cosine similarity between different approximations of the Gauss--Newton (GN) component of the Hessian and its true value for different datasets and architectures. Bottom: Similar plot showing the cosine similarity between different approximations of the Adagrad preconditioner matrix and its true value. As can be seen, $\text{Shampoo}^2$ tracks the optimal Kronecker approximation much more closely than Shampoo does. MNIST-2 refers to a binary subsampled MNIST dataset. For more details about datasets and architectures, please refer to Appendix~\ref{app:exp}.}
\label{fig:main}
\end{figure}

Second-order optimization is a rich research area within deep learning that has seen multiple influential works over the past few decades. Recently, these methods have seen success in practical large scale training runs such as Gemini 1.5 Flash~\citep{gemini15} and in academic benchmarks~\citep{dahl2023benchmarking}. One of the primary challenges in this field arises from the substantial memory and computational demands of traditional second-order methods, such as Adagrad \citep{duchi11} and Newton’s method. In the context of neural networks, both of these methods require storing and inverting a $|P| \times |P|$ dimensional matrix $H$ (either covariance of the gradients for Adagrad or the Gauss--Newton component of the Hessian for Newton's method), where $|P|$ represents the number of parameters of the neural network. With modern deep learning architecture scaling to billions of parameters, these requirements make the direct application of these methods impractical. To address this issue, various approaches have been proposed, including Hessian-free optimization \citep{martens2010deep} and efficient approximations of the matrix $H$ \citep{gupta2018shampoo,martens2015optimizing}. These methods aim to leverage second-order information while mitigating the computational and memory overhead.

The class of methods for efficiently approximating the matrix $H$ predominantly involve either a diagonal or a layer-wise Kronecker product approximation of $H$. These choices are motivated by the fact that, compared to maintaining the matrix $H$, both diagonal and layer-wise Kronecker products are significantly more memory-efficient to store and computationally efficient to invert. Two of the most well-known methods that utilize a layer-wise Kronecker product approximation of $H$ are K-FAC \citep{martens2015optimizing} and Shampoo \citep{gupta2018shampoo}.

In this work, we primarily focus on the Shampoo optimizer \citep{gupta2018shampoo}, which has recently gained increasing attention from the research community. Notably, in a recent benchmark of optimization algorithms proposed for practical neural network training workloads \citep{dahl2023benchmarking}, Shampoo appears to outperform all other existing methods. Another recent study, elucidating the Google Ads recommendation search pipeline, revealed that the Google Ads CTR model is trained using the Shampoo optimizer \citep{anil2022factory}. Additionally, a recent work \citep{shi2023distributed} implemented a distributed data parallel version of Shampoo, demonstrating its superior speed in training ImageNet compared to other methods.

Previously, Shampoo's approximation was shown to be an upper bound (in spectral norm) on the matrix $H$ \citep{gupta2018shampoo}. In this work, we make this connection much more precise. Prior research has established the notion of the optimal Kronecker product approximation (in Frobenius norm) of $H$ \citep{koroko2023efficient}, which can be obtained numerically using a power iteration scheme. The primary contribution of this work is to theoretically and empirically demonstrate that the square of the approximation used by Shampoo is nearly equivalent to the optimal Kronecker factored approximation of $H$.

The main contributions of the work are summarized below:

\begin{itemize}
    \item We theoretically show (Proposition~\ref{prop:main}) that the square of the Shampoo's approximation of $H$ is precisely equal to one round of the power iteration scheme for obtaining the optimal Kronecker factored approximation of the matrix $H$. Informally, for any covariance matrix $H = \E[gg^T]$ where $g \in \mathbb{R}^{mn}$ \footnote{Gauss--Newton component of the Hessian can also be expressed as a covariance matrix. For details, refer Section \ref{sec:tech_hess}}, we argue that the \textit{right} Kronecker product approximation of $H$  is $\mathbb{E}[ G G^\top ] \otimes \mathbb{E}[ G^\top G ]$ while Shampoo proposes $\mathbb{E}[ G G^\top ]^{1/2} \otimes \mathbb{E}[ G^\top G ]^{1/2}$, with $G \in \mathbb{R}^{m \times n}$ representing a reshaped $g$ into a matrix of size $m \times n$. 
    \item We empirically establish that the result of one round of power iteration is very close to the optimal Kronecker factored approximation (see Figure \ref{fig:main}), and provide theoretical justification for the same.
    \item For the Hessian based viewpoint of Shampoo (Section \ref{sec:tech_hess}), we empirically demonstrate the impact on the Hessian approximation of various practical tricks implemented to make Shampoo more computationally efficient such as averaging gradients over batch (Section~\ref{sec:avg}) and using empirical Fisher instead of the actual Fisher (Section~\ref{sec:labels}).
\end{itemize}

\textbf{Remark.} Previous works \citep{balles2020geometry, lin2024remove} have explored the question of why Adagrad-based approaches like Adam and Shampoo have an extra square root compared to the Hessian inverse in their update. This alternative question is orthogonal to our contribution. For details, refer Appendix \ref{app:sq_root}.

\textbf{Paper organization.} In Section \ref{sec:tech}, we cover the technical background necessary for understanding this work. In Section \ref{sec:opt_kron}, we provide a general power iteration scheme for obtaining the optimal Kronecker product approximation of the matrix $H$, and establish the the connection between Shampoo's approximation and the optimal Kronecker product approximation of $H$. In Section \ref{sec:hess_shampoo}, we explore the Hessian approximation viewpoint of Shampoo and empirically study how various practical tricks to make Shampoo more computationally efficient impact the quality of the Hessian approximation. In Section \ref{sec:rel_work}, we cover closely related works and conclude with discussing the limitations of the work in Section \ref{sec:limitations}. In Appendix~\ref{app:vit}, we include additional experiments on the ViT architecture and compare with the K-FAC approximation to the Hessian. Detailed related work, proofs, dataset and architecture details have been deferred to the Appendix.
\section{Technical background}
\label{sec:tech}

We use lowercase letters to denote scalars and vectors, and uppercase letters to denote matrices. For a symmetric matrix $A$, $A \succcurlyeq 0$ (resp. $A \succ 0$) denotes that $A$ is positive semi-definite (resp. positive definite). Similarly, for symmetric matrices $A$ and $B$, $A \succcurlyeq B$ (resp. $A \succ B$) denotes $A - B \succcurlyeq 0$ (resp. $A - B \succ 0$). We will use $M[i, j]$ refer to the 0-indexed $(i, j)$ entry of the matrix $M$. The Kronecker product of two matrices $A \in \mathbb{R}^{p \times q}$ and $B \in \mathbb{R}^{r \times s}$ is denoted by $A \otimes B \in \mathbb{R}^{pr \times qs}$. It is defined such that 
$(A\otimes B)[ri+i',sj+j']=A[i, j]B[i', j']$ where $0 \leq i < p, 0 \leq j < q, 0 \leq i' < r, 0 \leq j' < s$. Vectorization of a matrix $A \in \mathbb{R}^{m \times n}$, denoted by $\vect{A}$, is a $mn$-dimensional column vector obtained by stacking the columns of $A$ on top of one another. We will usually denote $\vect{A}$ by $a$.

Following is a basic lemma about Kronecker products that will be used later
\begin{lemma}[\citet{kronecker}]\label{lem:kronecker_matrix_relationships}
$(A \otimes B) \vect{G} = \vect{BGA^\top}$.
\end{lemma}

\subsection{Shampoo} \label{sec:tech_shampoo}

The original Shampoo \citep{gupta2018shampoo} paper introduced its algorithm  as an approximation of an online learning algorithm Adagrad \citep{JMLR:v12:duchi11a}. Shampoo can also be interpreted~\citep{anil2021towards,asdl} as approximating the Gauss--Newton component of the Hessian. Both of these perspectives will be discussed in Section~\ref{sec:tech_ada} and ~\ref{sec:tech_hess} respectively.
. 
\subsubsection{Adagrad based perspective of Shampoo}\label{sec:tech_ada}

\textbf{Adagrad:} This is a preconditioned online learning algorithm, that uses the accumulated covariance of the gradients as a preconditioner. Let $\theta_t \in \mathbb{R}^p$ denote the parameters at time $t$ and let $g_t \in \mathbb{R}^p$ denote the gradient. It maintains a preconditioner $H_{\text{Ada}} = \sum_{t=1}^T g_tg_t^\top$. The update for the parameter for learning rate $\eta$ are given by

\[\theta_{T+1} = \theta_T - \eta H_{\text{Ada}}^{-1/2}g_T. \]

Shampoo is a preconditioned gradient method which maintains a layer-wise Kronecker product approximation to full-matrix Adagrad. Let the gradient for a weight matrix\footnote{We will focus on weights structured as matrices throughout this paper.} $W_t \in \mathbb{R}^{m \times n}$ at time $t$ be given by $G_t \in \mathbb{R}^{m \times n}$. The lemma below is used to obtain the Shampoo algorithm from Adagrad:

\begin{lemma}[\citet{gupta2018shampoo}]
\label{lem:gupta_adagrad}
    Assume that $G_1,...,G_T$ are matrices of rank at most $r$. Let $g_t = \vect{G_t}$ for all $t$. Then, with $\preccurlyeq$ representing the  for any $\epsilon > 0$,
    \[ \epsilon I_{mn} + \frac{1}{r} \sum_{t=1}^T g_t g_t^\top \preccurlyeq \left(\epsilon I_m + \sum_{t=1}^T G_t G_t^\top\right)^{1/2} \otimes \left(\epsilon I_n + \sum_{t=1}^T G_t^\top G_t\right)^{1/2}. \]
\end{lemma}

Based on the above lemma, Shampoo maintains two preconditioners $L_t \in \mathbb{R}^{m \times m}$ and $R_t \in \mathbb{R}^{n \times n}$, which are initialized to $\epsilon I_m$ and $\epsilon I_n$ respectively. . The update for the preconditioners and the Shampoo update for a learning rate $\eta$ is given by

\[ L_T = L_{T-1} + G_TG_T^\top; \quad R_T = R_{T-1} + G_T^\top G_T; \quad W_{T+1} = W_T - \eta L_T^{-1/4} G_T R_T^{-1/4}. \]

In Lemma~\ref{lem:gupta_adagrad} the matrix $H_{\text{Ada}} = \sum_{t=1}^T g_tg_t^\top$ is approximated (ignoring $\epsilon$ and scalar factors) by the the Kronecker product $\left(\sum_{t=1}^T G_t G_t^\top\right)^{1/2} \otimes \left(\sum_{t=1}^T G_t^\top G_t\right)^{1/2}$. Our main focus will be to study the \textit{optimal Kronecker product approximation} of the matrix $H_{\text{Ada}}$ and its connection to Shampoo's approximation (done in Section~\ref{sec:opt_kron}).

\subsubsection{Hessian based perspective of Shampoo}\label{sec:tech_hess}

In this section we describe the Hessian approximation viewpoint of Shampoo explored by previous works \citep{anil2021towards,asdl} as an alternative to the Adagrad viewpoint described above. Our theoretical and empirical results hold for both viewpoints.

\textbf{Gauss--Newton (GN) component of the Hessian.}
For a datapoint $(x,y)$, let $f(x)$ denote the output of a neural network and $\mathcal{L}(f(x),y)$ represent the training loss. Let $W \in \mathbb{R}^{m \times n}$ represent a weight matrix in the neural network and $\mathcal{D}$ denote the training distribution. Then, for CE loss, the Gauss-Newton component of the Hessian of the loss with respect to $W$ is given by (see Appendix~\ref{app:tech_hess} for details)

\[ H_{\text{GN}} = \E_{(x,y) \sim \mathcal{D}}\left[\frac{\partial f}{\partial W} \frac{\partial^2 \mathcal{L}}{\partial f^2} \frac{\partial f}{\partial W}^\top\right] = \E_{\substack{x \sim \mathcal{D}_x \\ s \sim f(x)}} \left[g_{x,s} g_{x,s}^\top \right] , \]

where, for brevity, $f(x)$ denotes the output distribution of the neural network and $\mathcal{D}_x$ represents the training distribution of $x$ \citep{DBLP:journals/corr/abs-1301-3584}. The right-hand side of the equation is also referred to in the literature as the Fisher matrix, and its counterpart for real labels, $\E_{(x,y) \sim \mathcal{D}}\left[g_{x,y} g_{x,y}^\top \right]$, is referred to as the empirical Fisher. For brevity, going forward, we will assume that $x$ is drawn from $\mathcal{D}_x$ and represent the Fisher matrix as $\E_{x, s \sim f(x)} \left[g_{x,s} g_{x,s}^\top \right]$. Similarly, when both $x$ and $y$ are used, we will assume they are drawn from $\mathcal{D}$.

The aim of algorithms such as K-FAC and Shampoo (when viewed from the Hessian perspective) is to do a layerwise Kronecker product approximation of the Fisher matrix $H_{\text{GN}}$. The following lemma establishes the approximation made by Shampoo:

\begin{lemma}[Adapted from \citet{gupta2018shampoo,anil2021towards}] \label{lem:shamp_approx}
    Assume that $G_{x,s}$ are matrices of rank at most $r$. Let $g_{x,s} = \vect{G_{x,s}}$ . Then, for any $\epsilon > 0$,
    \begin{equation} \label{eq:shamp_approx}
    \E_{x, s \sim f(x)} \left[g_{x,s} g_{x,s}^\top \right] \preccurlyeq r\left(\E_{x, s \sim f(x)} \left[G_{x,s} G_{x,s}^\top \right] \right)^{1/2} \otimes \left(\E_{x, s \sim f(x)} \left[G_{x,s}^\top G_{x,s} \right] \right)^{1/2}.
    \end{equation}
\end{lemma}

In Lemma~\ref{lem:gupta_adagrad} the matrix on the left hand side is equal to $H_{\text{GN}}$ and the right hand side represents the $H_{\text{GN}}$ approximation made by Shampoo. However, computing this approximation at every step is expensive. So, in practice, Shampoo makes two additional approximations on top.

First, it replaces the per-input gradient by batch gradient, i.e, replaces $\E_{x, s \sim f(x)} [G_{x,s} G_{x,s}^\top]$ by $\E_{B,\bf{s}} [G_{B,\bf{s}}  G_{B,\bf{s}}^\top]$, where $B$ denotes the batch, $\bf{s}$ is the concatenation of $s \sim f(x)$ for all $(x,y) \in B$ and $G_{B, \bf{s}} = \frac{1}{|B|} \sum_{(x,y) \in B, s=\mathbf{s} [x]} G_{x,s}$ is the \textit{sampled batch gradient}, with $\mathbf{s} [x]$ representing the sampled label corresponding to $x \in B$.

Second, it replaces sampled labels with real labels, i.e., it replaces $\E_{B,\bf{s}} [G_{B,\bf{s}} G_{B,\bf{s}}^\top]$ with $\E_{B} [G_{B} G_{B}^\top]$, where $G_{B} = \frac{1}{|B|} \sum_{(x,y) \in B} G_{x,y}$ is the \textit{batch gradient}.

Thus, if $G_j$ and $W_j$ represent the batch gradient and weight matrix at iteration $j$, and $\lambda$ is an exponential weighting parameter, then the update of Shampoo is given by
\[ L_j = \lambda L_{j-1} + (1 - \lambda) G_j G_j^\top; \quad R_j = \lambda R_{j-1} + (1 - \lambda) G_j^\top G_j; \quad W_{j+1} = W_j - \eta L_j^{-1/4} G_j R_j^{-1/4}, \]
where $L_j$ and $R_j$ represent the left and right preconditioners maintained by Shampoo, respectively.

Our focus (when viewing Shampoo from the Hessian perspective) will be to study
\begin{itemize}
    \item The optimal Kronecker product approximation of the matrix $H_{\text{GN}}$ and its connection to Shampoo's approximation (done in Section~\ref{sec:opt_kron}).
    \item The effect of the aforementioned two approximations on the approximation quality (done in Section~\ref{sec:hess_shampoo}).
\end{itemize}

\subsection{Optimal Kronecker product approximation} \label{sec:tech_opt_kron}
 For Frobenius norm (or other ``entry-wise'' matrix norms), finding the optimal Kronecker product approximation of a matrix $H \in \mathbb{R}^{mn \times mn}$ is equivalent to finding the optimal rank-one approximation of a rearrangement of $H$. We define the rearrangement operator $\reshape{}$, applied to a matrix $H$ such that, $$\reshape{H}[mi+i', nj+j'] = H[mj+i, mj'+i'],$$
where $\{i,i'\} \in [0,1,...,m-1]$, $\{j,j'\} \in [0,1,...,n-1]$ and $\reshape{H} \in \mathbb{R}^{m^2 \times n^2}$. A property of $\reshape$ that will be useful to us is:
\begin{equation}
\label{eq:two_spaces}
   H = A \otimes B \iff \reshape{H} = ab^\top,
\end{equation}
where $A \in \mathbb{R}^{m \times m}$, $a = \vect{A} \in \mathbb{R}^{m^2}$, $B \in \mathbb{R}^{n \times n}$ and $b = \vect{B} \in \mathbb{R}^{n^2}$. This property can be used to prove the following result on optimal Kronecker product approximation:

\begin{lemma}[\citet{approximation_with_kronecker}]\label{lem:kronecker_approx}
Let $H \in \mathbb{R}^{mn \times mn}$ be a matrix and let $L \in \mathbb{R}^{m \times n}, R \in \mathbb{R}^{n \times m}$. Then, the equivalence of the Kronecker product approximation of $H$ and the rank-one approximation of $\reshape{H}$ is given by:
\[
\| H - L \otimes R \|_F = \| \reshape{H} - \vect{L} \vect{R}^\top \|_F,
\]
where $\| \cdot \|_F$ denotes the Frobenius norm.
\end{lemma}

Since the optimal rank-1 approximation of a matrix is given by its singular value decomposition (SVD), we conclude:
\begin{corollary}
Let $H \in \mathbb{R}^{mn \times mn}$. If the top singular vectors and singular value of $\reshape{H}$ are represented by $u_1, v_1$ and $\sigma_1$, respectively, then the matrices $L \in \mathbb{R}^{m \times m}$ and $R \in \mathbb{R}^{n \times n}$ defined by
\[
\operatorname{vec}(L) = \sigma_1 u_1, \quad \operatorname{vec}(R) = v_1,
\]
minimize the Frobenius norm $\|H - L \otimes R\|_F$.
\end{corollary}

\textbf{Obtaining SVD by power iteration.}
Power iteration \citep{GoluVanl96} is a well-known method for estimating the top eigenvalue of a matrix $M$. It can also be specialized for obtaining the top singular vectors of a matrix. The corresponding iterations for the left singular vector $\ell$ and the right singular vector $r$ are given by
\begin{equation} \label{eq:svd}
    \ell_k \leftarrow Mr_{k-1} ; \quad r_k \leftarrow M^\top \ell_{k-1},
\end{equation} 
where $k$ denotes the iteration number.

\textbf{Cosine similarity.}
We will be using cosine similarity between matrices as a metric for approximation. For two matrices $M_1$ and $M_2$, this refers to $\text{Tr}(M_1M_2^\top) / (||M_1||_F \cdot || M_2 ||_F)$. A value of 1 indicates perfect alignment, while a value of 0 indicates orthogonality.

\section{Optimal Kronecker product approximation and Shampoo} \label{sec:opt_kron}

In this section, we will specialize the theory of Section \ref{sec:tech_opt_kron} for finding the optimal Kronecker product approximation of a covariance matrix $H = \E_{g \sim \mathcal{D}_g}[gg^\top]$ for $g \in \mathbb{R}^{mn}$. Both perspectives of Shampoo described in Section~\ref{sec:tech_shampoo} are concerned with Kronecker product approximations of $H$ of the form $L \otimes R$ where $L \in \mathbb{R}^{m \times m}, R \in \mathbb{R}^{n \times n}$, but for different distributions $\mathcal{D}_g$. For the Adagrad viewpoint, with $\mathcal{D}_g$ as the uniform distribution over $g_t$ where $1 \leq t \leq T$ refers to the gradient at timestep $t$, $H = H_{\text{Ada}}$. For the Hessian viewpoint, with $\mathcal{D}_g$ as the distribution over gradients with batch size 1 and with sampled labels, $H = H_{\text{GN}}$ (see Section~\ref{sec:tech_hess} for derivation). %

Since our results will hold for all distributions $\mathcal{D}_g$, we will use $\E[gg^\top]$ to refer to $\E_{g \sim \mathcal{D}_g}[gg^\top]$ to simplify notation. The main goal of this section will be to study the optimal Kronecker product approximation to such a generic matrix $H$, see its connection to Shampoo, and experimentally validate our results for $H = H_{\text{Ada}}$ and $H = H_{\text{GN}}$, which are described in Section~\ref{sec:tech_ada} and \ref{sec:tech_hess}, respectively.

\citet{van1993approximation} describe an approach to find the optimal Kronecker product approximation of a matrix (with respect to the Frobenius norm). \citet{koroko2023efficient} use this approach to find the optimal layer-wise Kronecker product approximation of the hessian matrix for networks without weight sharing. We will now do a general analysis which would also be applicable to neural networks with weight sharing.

Since $g \in \mathbb{R}^{mn}$, each entry of $g$ can be described as a tuple $(i,j) \in [m] \times [n]$. Consequently, every entry of $H$ can be represented by the tuple $((i, j), (i', j'))$. We now consider the matrix $\hat{H} \coloneq \reshape{H} \in \mathbb{R}^{m^2 \times n^2}$, which is a rearrangement (see Section~\ref{sec:tech}) of the entries of $H$.

By using equation~\ref{eq:two_spaces} we get that:
\[ \hat{H} = \mathbb{E} [ G \otimes G ]. \]
Further, by Lemma~\ref{lem:kronecker_approx}, we have that if $L \otimes R$ is the optimal Kronecker product approximation of $H$, then $\ell r^\top$ is the optimal rank-1 approximation of $\hat{H}$, where $\ell = \vect{L}$ and $r = \vect{R}$. Hence, the problem reduces to finding the optimal rank-1 approximation of $\hat{H}$. Applying the power iteration scheme described in Equation \ref{eq:svd} for estimating the top singular vectors of $\hat{H}$ and using Lemma~\ref{lem:kronecker_matrix_relationships} yields (where $k$ denotes the $k^{\text{th}}$ step of power iteration): 
\begin{align*}
    \ell_k &\leftarrow \hat{H} r_{k-1} =  \mathbb{E} [ G \otimes G ] r_{k-1} = \vect{\mathbb{E}[ G R_{k-1} G^\top ]}, \\
    r_k &\leftarrow \hat{H}^\top \ell_{k-1} = \mathbb{E} [ G \otimes G ]^\top \ell_{k-1} = \vect{\mathbb{E}[ G^\top L_{k-1} G ]}.
\end{align*}
Reshaping vectors on both sides into matrices results in:
\begin{equation}
\label{eq:pow_it}
L_k \leftarrow \mathbb{E}[ G R_{k-1} G^\top ]; \quad
R_k \leftarrow \mathbb{E}[ G^\top L_{k-1} G ].
\end{equation}

\subsection{One round of power iteration}
\label{sec:one_round}
Our first and main approximation involves replacing the iterative power iteration scheme (Equation~\ref{eq:pow_it}) with just a single iteration. This leads to the main contribution of our work:

\begin{proposition}
\label{prop:main}
    One step of power iteration, starting from the identity, for obtaining the optimal Kronecker product approximation of $H$ is precisely equal to the square of the Shampoo's approximation of $H$
\end{proposition}
\begin{proof}
   The initialization for the single iteration will use the identity matrix, i.e., $I_m$ and $I_n$ for $L$ and $R$, respectively. Thus, we transition from the iterative update equations:
\[
    L_k \leftarrow \mathbb{E}[ G R_{k-1} G^\top ]; \quad
R_k \leftarrow \mathbb{E}[ G^\top L_{k-1} G ],
\]
to the simplified single-step expressions:
\[
   L \leftarrow \mathbb{E}[ G G^\top ]; \quad 
R \leftarrow \mathbb{E}[ G^\top G ].
\] 
With the above expression for $L$ and $R$, $L \otimes R$ is precisely equal to the \textit{square} of the Shampoo's approximation of $H$ given by the right hand side of Equation \ref{eq:shamp_approx}.
\end{proof}
As shown in Figure \ref{fig:main}, for various datasets and architectures, this single step of power iteration is very close to the optimal Kronecker product approximation for both $H = H_{\text{GN}}$ (top) and $H = H_{\text{Ada}}$ (bottom). However, we can see that the upper bound proposed by the original Shampoo work \citep{gupta2018shampoo} is significantly worse.

\subsubsection{Why initialize with the identity matrix?}
\label{sec:whyI}

\begin{figure}[htbp]
\centering
\begin{tabular}{cccc}
       & \textnormal{\small {MNIST-2}} & \textnormal{\small {CIFAR-5M}} & \textnormal{\small {ImageNet}} \\
\rotatebox[origin=c]{90}{\textnormal{\small {Gauss--Newton}}} & 
\animage{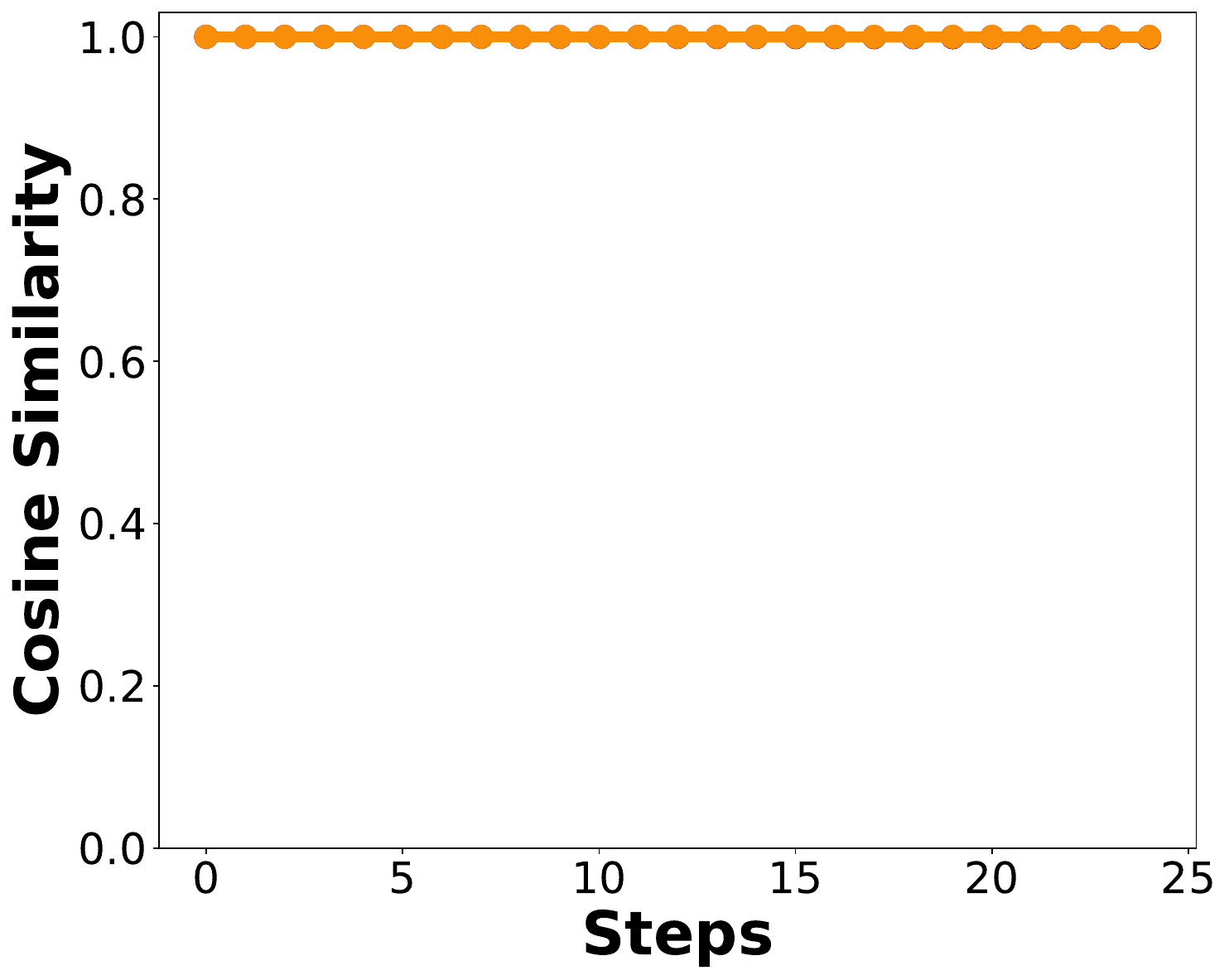} & 
\animage{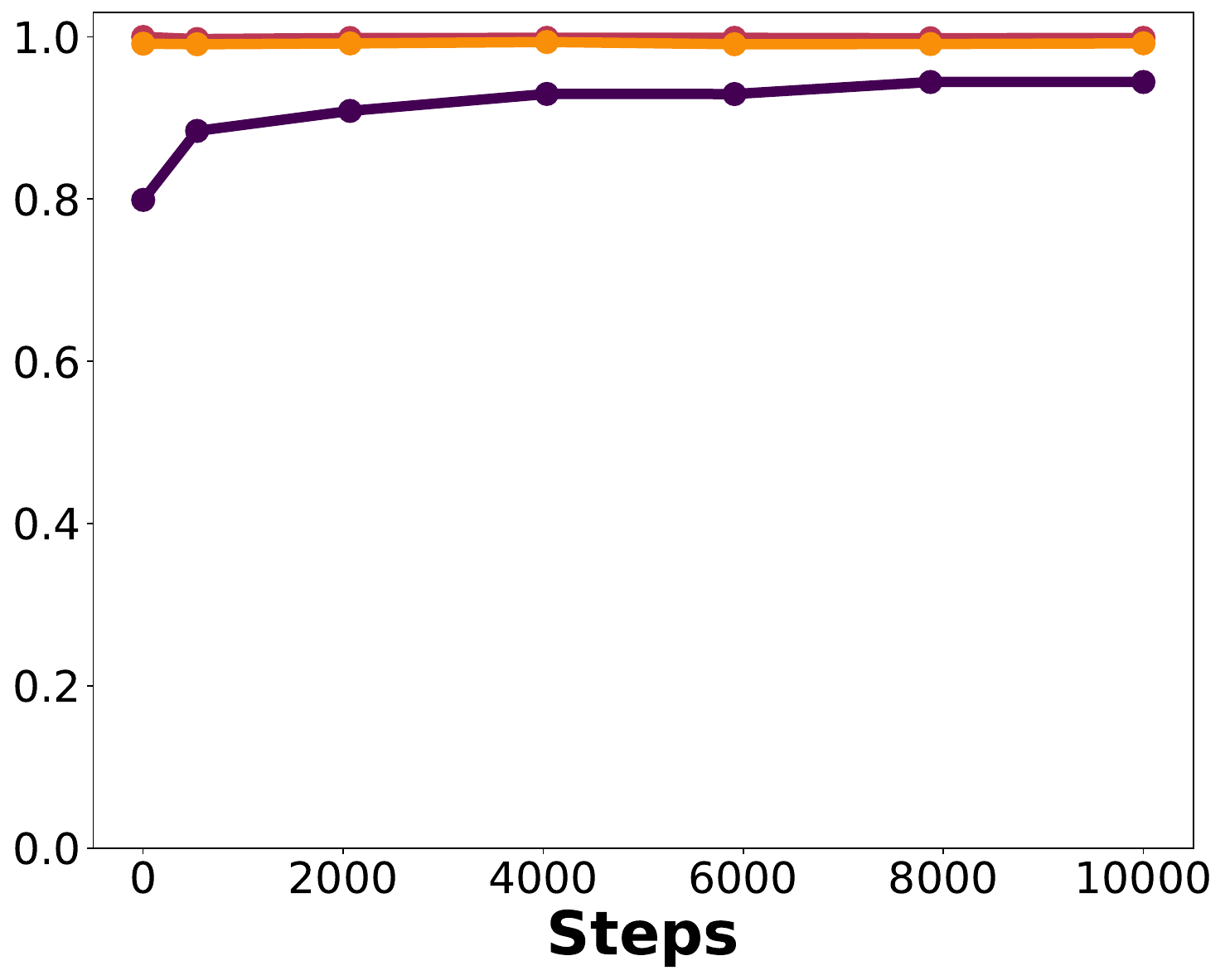} & 
\animage{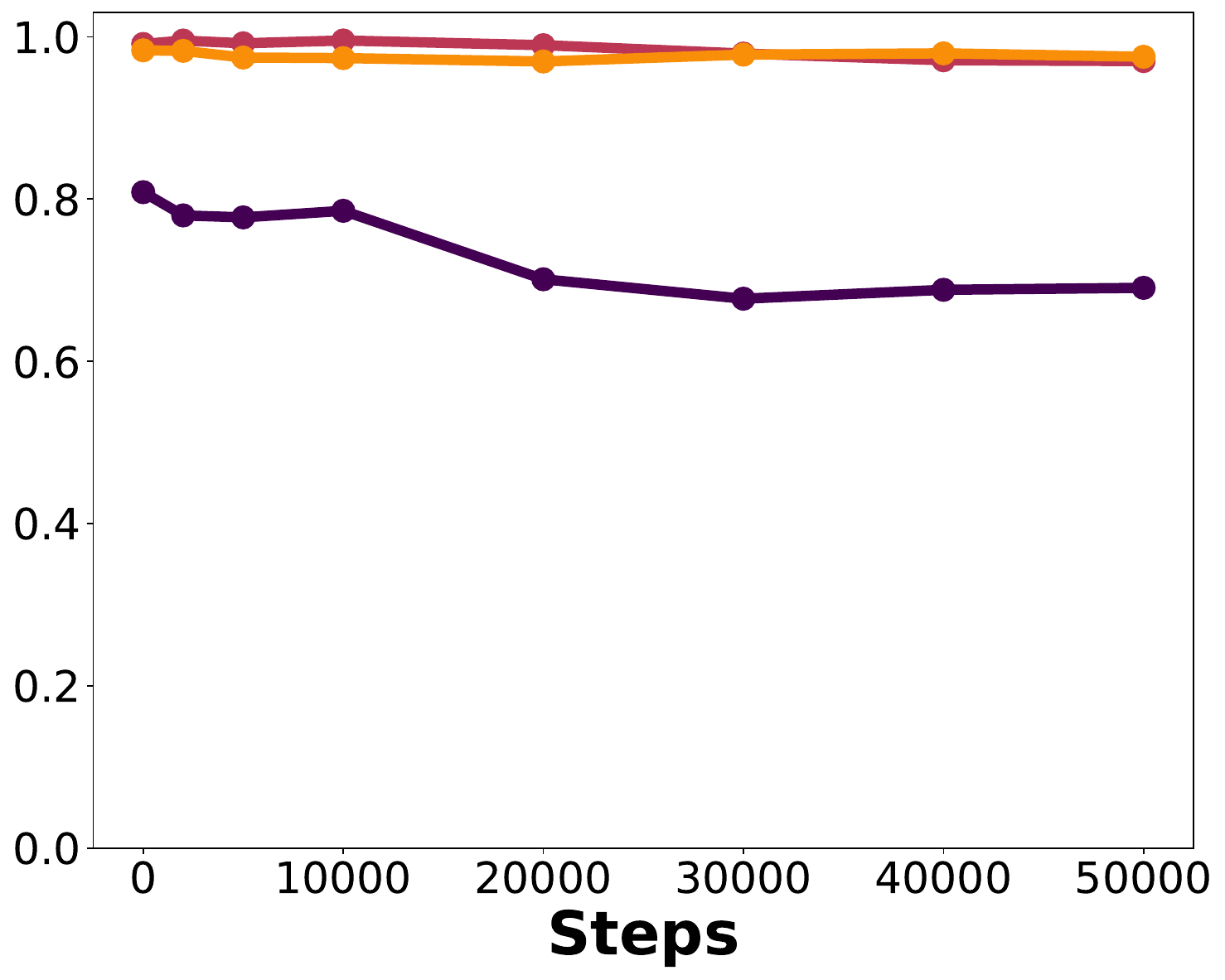} \\
\rotatebox[origin=c]{90}{\textnormal{\small {Adagrad}}} & 
\animage{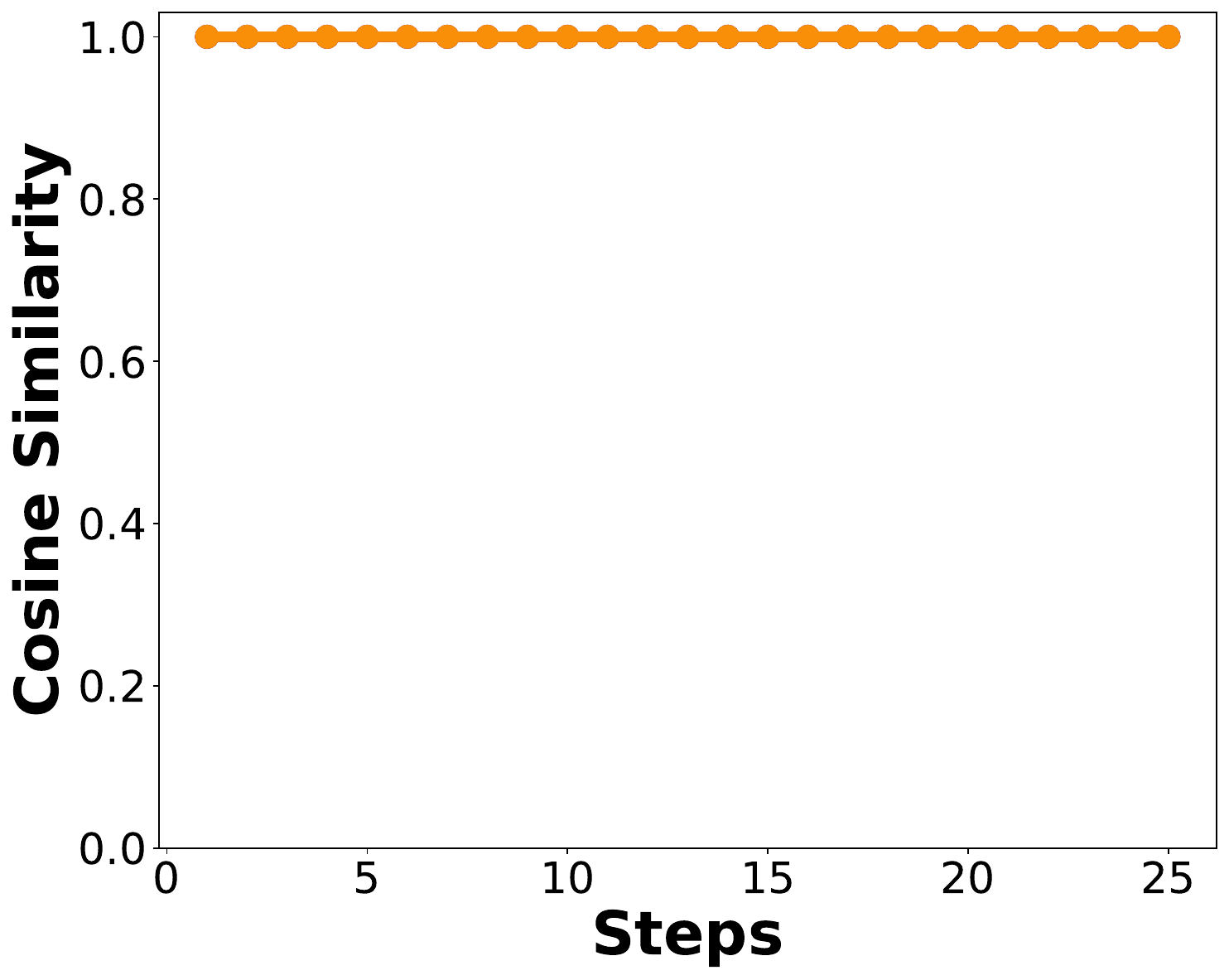} & 
\animage{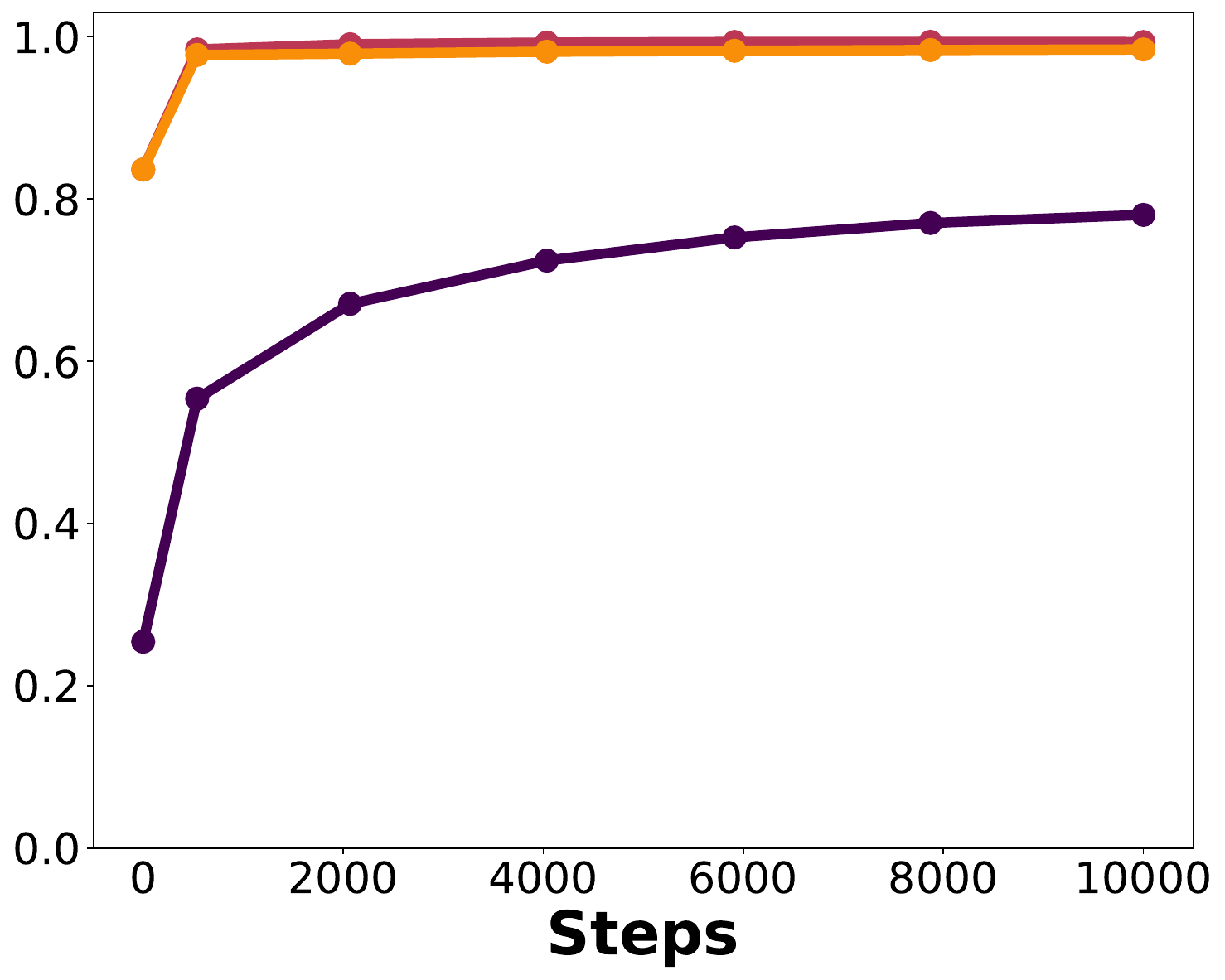} & 
\animage{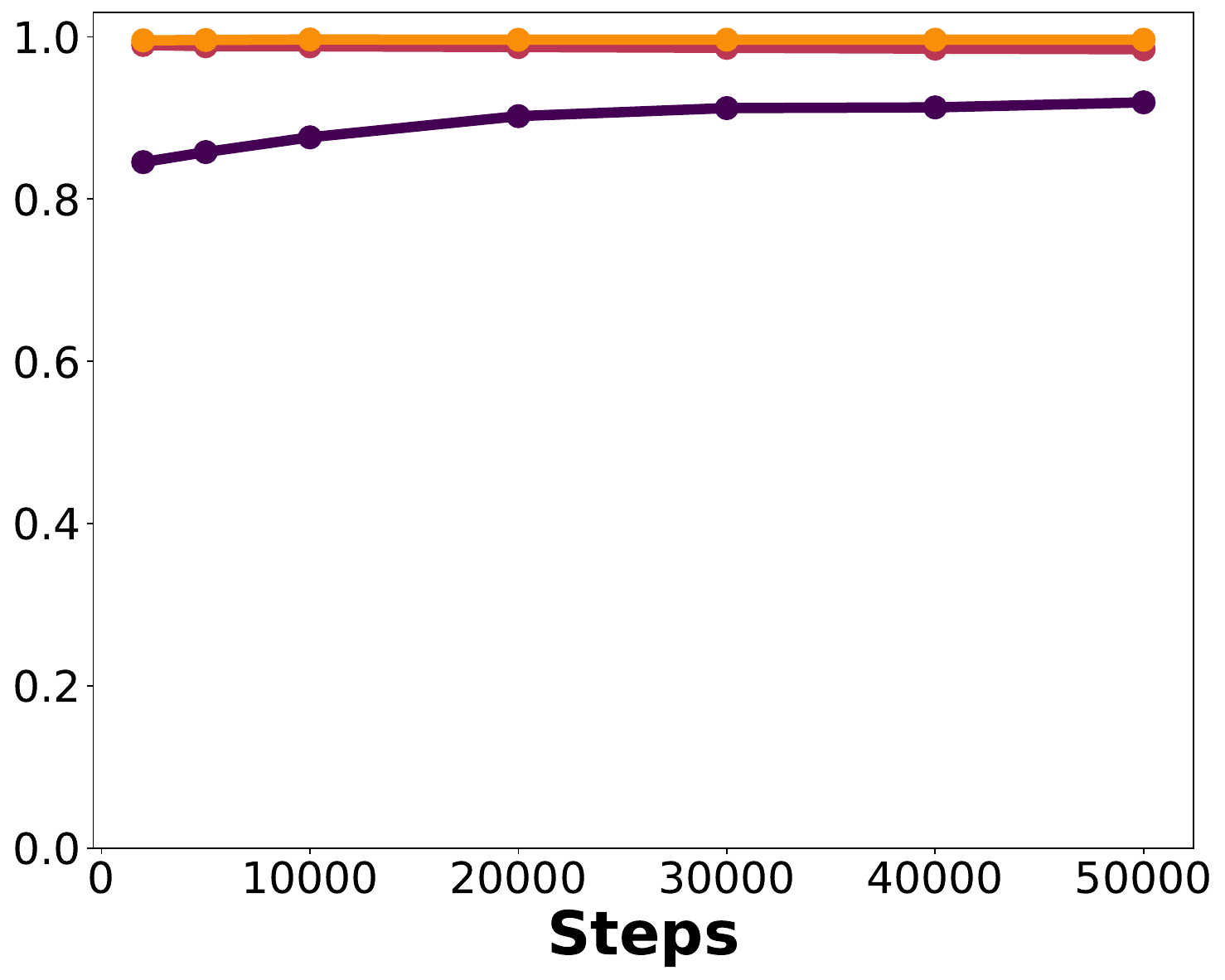} \\
\end{tabular}

\begin{minipage}{\textwidth}
    \centering
    \includegraphics[width=0.95\textwidth]{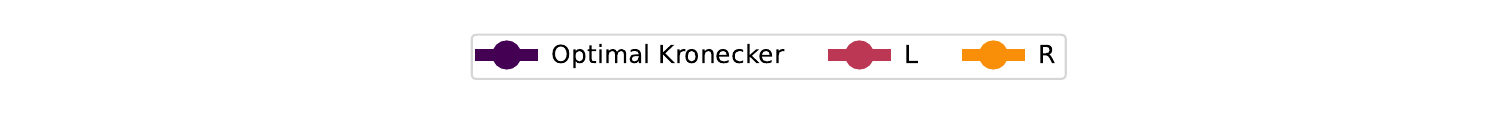}
\end{minipage}
\caption{Comparing $\frac{\sigma_1}{\sqrt{\sum_i \sigma_i^2}}$ and $\frac{\alpha_1 \sigma_1}{\sqrt{\sum_i \alpha_i^2 \sigma_i^2}}$ for various datasets and architectures. The top row is for $H = H_{\text{GN}}$ while the bottom row is for $H = H_{\text{Ada}}$. The $L$ and $R$ legends represent $\frac{\alpha_1 \sigma_1}{\sqrt{\sum_i \alpha_i^2 \sigma_i^2}}$ for the left and right singular vector respectively. The ``Optimal Kronecker'' legend represents $\frac{\sigma_1}{\sqrt{\sum_i \sigma_i^2}}$ (see Section~\ref{sec:whyI}). As seen, $\frac{\alpha_1 \sigma_1}{\sqrt{\sum_i \alpha_i^2 \sigma_i^2}}$ is much closer to $1$ as compared to $\frac{\sigma_1}{\sqrt{\sum_i \sigma_i^2}}$, demonstrating the role played by identity initialization in ensuring convergence of power iteration in one round. See Appendix~\ref{app:fig_details} for details.}
\label{fig:whyI}
\end{figure}

Suppose the SVD of $\hat{H}$ is given by $\hat{H} = \sum_i \sigma_i u_iv_i^T$ , or equivalently, $H = \sum_i \sigma_i U_i \otimes V_i$.  The convergence of the power iteration in one step depends on the inner product of the initialization vector with the top singular vector. Let us focus on the left side,\footnote{The discussion for the other side is analogous.} i.e., the update $L  \leftarrow \mathbb{E}[ G G^\top ]$ which as described earlier is equivalent to starting with the initialization $I_n$. Let $\vectb{I_n} = \sum_i \alpha_i v_i$ i.e. $I_n = \sum_i \alpha_i V_i$. After one iteration, we obtain $\ell := \sum_i \alpha_i \sigma_i u_i$, and correspondingly, $L  := \sum_i \alpha_i \sigma_i U_i$. We are interested in assessing how closely $\ell$ approximates the leading eigenvector $u_1$. The cosine similarity between $\ell$ and $u_1$ is given by $\frac{\alpha_1 \sigma_1}{\sqrt{\sum_i \alpha_i^2 \sigma_i^2}}$.

One reason why the cosine similarity might be large is that $\hat{H}$ is nearly rank-1 ($\sigma_1$ is large); that is, $H$ is closely approximated by a Kronecker product. As illustrated in Figure~\ref{fig:main}, this assumption does not universally hold. Instead, we propose an alternative explanation for why a single step of power iteration is typically sufficient: the coefficient $\alpha_1$ is usually larger than $\alpha_i$ for all $i \geq 2$. We begin by providing a theoretical justification for this, followed by empirical evidence from our experiments.

We start by noting that $\alpha_i = \vectb{I_n}^Tv_i = \text{Tr}(V_i)$.
Now, we will show that using the identity matrix as initialization is a good choice since a)  shows it has the maximal dot product with possible top components i.e., PSD matrices (Proposition~\ref{lem:iden}), and b) we expect it to have a small dot product with later components.
\begin{restatable}[~\citet{van1993approximation}]{lemma}{lempsd}
\label{claim:u_psd}  
  $V_1$ is a Positive Semi-Definite (PSD) matrix.
\end{restatable}
Since $V_1$ is a PSD matrix we would like to initialize our power iteration with a matrix which is close to all PSD matrices. Now, we will show that identity is the matrix which achieves this, specifically it maximizes the minimum dot product across the set of PSD matrices of unit Frobenius norm.
\begin{restatable}{proposition}{lemiden}
\label{lem:iden}  
Consider the set of PSD matrices of unit Frobenius norm of dimension $m$ denoted by $S_m$. Then
    \[ \frac{1}{\sqrt{m}} I_m = \argmax_{M \in S_m} \min_{M' \in S_m} \langle \vect{M}, \vect{M'} \rangle . \]
\end{restatable}

The previous proposition argues that $I_m$ maximizes the worst-case dot product with possible top singular vectors. Now, we argue that its dot product with other singular vectors should be lower.
\begin{restatable}{lemma}{lemnopsd}
\label{claim:no_psd}  
  If $V_1$ is positive-definite, then $V_i$ for $i \geq 2$ are not PSD.
\end{restatable}
Therefore, the diagonal elements of $V_i$ for $i \geq 2$ need not be positive, and this might lead to cancellations (for $i \geq 2$) in the trace of $V_i$ which is equal to $\alpha_i$. Hence we expect $\alpha_i$'s for $i \geq 2$ to be smaller than $\alpha_1$. We now show experiments to demonstrate this in practice. To quantify the benefit of $\alpha_1$ usually being larger than $\alpha_i$ for $i \geq 2$, we will compare $\frac{\alpha_1 \sigma_1}{\sqrt{\sum_i \alpha_i^2 \sigma_i^2}}$ (for both left and right singular vectors) and $\frac{\sigma_1}{\sqrt{\sum_i \sigma_i^2}}$. The latter can be interpreted as the cosine similarity if all $\alpha$'s were equal or as a measure of how close $\hat{H}$ is to being rank 1 since it is equal to the cosine similarity between $u_1v_1^T$ and $\hat{H}$. Thus $\frac{\sigma_1}{\sqrt{\sum_i \sigma_i^2}}$ is equal to the ``Optimal Kronecker'' cosine similarity used in Figure~\ref{fig:main}. In Figure~\ref{fig:whyI} we track both of these quantities through training and indeed observe that $\frac{\alpha_1 \sigma_1}{\sqrt{\sum_i \alpha_i^2 \sigma_i^2}}$ are significantly closer to 1 than $\frac{\sigma_1}{\sqrt{\sum_i \sigma_i^2}}$ for both $H = H_{\text{GN}}$ (top) and $H = H_{\text{Ada}}$ (bottom).

\subsubsection{Exact Kronecker product structure in \texorpdfstring{$H$}{HGN}}

The previous discussion shows that $\E\left[G G^\top\right] \otimes \E\left[G^\top G\right]$ is close to the optimal Kronecker product approximation of $H$. In this section we will show that this holds exactly if $H$ is a Kronecker product. Intuitively, this holds since if $H$ is a Kronecker product, then $\hat{H}$ is rank-1, and one round of power iteration would recover $\hat{H}$. Until now, we have been focusing on the direction of top singular vectors of $\hat{H}$, but with the assumption of $\hat{H}$ being rank 1, we can compute the explicit expression for $\hat{H}$, and hence of $H$.
\begin{restatable}{corollary}{corrrank}
    \label{corr:rank1}
    Under the assumption that $\hat{H}$ is rank-1,
\[ H = \left(\E\left[G G^\top\right] \otimes \E\left[G^\top G\right]\right) / \Tr\left(\E\left[G G^\top\right]\right). \]   
\end{restatable}
\begin{proof}
Let $\hat{H} = \sigma uv^\top$, i.e, $H = \sigma U \otimes V$. Let $I_m = \text{Tr}(U) U + R_m$ and $I_n = \text{Tr}(V) V + R_n$, where $R_m$ and $R_n$ are the residual matrices. Now, after one round of power iteration, the left and right estimates provided by Shampoo are given by
\[ \E\left[G G^\top\right] = \sigma \text{Tr}(V) U, \quad \E\left[G^\top G\right] = \sigma \text{Tr}(U) V. \]
From this, we can see that $\Tr\left(\E\left[G G^\top\right]\right) = \sigma \Tr(U) \Tr(V)$. Thus
\[ H = \sigma U \otimes V = \left(\E\left[G G^\top\right] \otimes \E\left[G^\top G\right]\right)/ \Tr\left(\E\left[G G^\top\right]\right). \]
\end{proof}

Since $H = \hat{H}_{\text{GN}}$ is an $m^2 \times 1$ matrix for binomial logistic regression, it is rank-1, so the equality in the corollary holds. In other words, the square of Shampoo's $H_{\text{GN}}$ estimate perfectly correlates with $H_{\text{GN}}$ for binomial logistic regression. This is demonstrated in the first plot of Figure \ref{fig:main}.

We note that $\left(\E\left[G G^\top\right] \otimes \E\left[G^\top G\right]\right) / \Tr\left(\E\left[G G^\top\right]\right)$ as an estimate of $H$ was also derived by~\citet{yi21}. But their assumptions were much stronger than ours, specifically they assume that the gradients follow a \textit{tensor-normal distribution}, which implies that $\hat{H}$ is rank 1. Instead, we only make a second moment assumption on the gradients: $H = \E [ gg^\top]$ is an exact Kronecker product. We also note that our derivation of the \textit{direction} $\E\left[G G^\top\right] \otimes \E\left[G^\top G\right]$ being close to the optimal Kronecker product approximation holds independently of $\hat{H}$ being rank 1.

\subsubsection{Discussion about optimization}

Let us refer to $\E[ GG^\top ] \otimes \E[ G^\top G]$ by $H_{1}$. As mentioned in Equation \ref{eq:shamp_approx}, the original Shampoo paper used the approximation $H$ used was $H_{1/2} \coloneq \E[ GG^\top ]^{1/2} \otimes \E[ G^\top G]^{1/2}$. In practice, when using Shampoo as an optimization algorithm, the gradient step is taken in the direction of $H_{1/2}^{-p} \nabla L$ where $p$ is tuned as a hyperparameter~\citep{anil2021towards,shi2023distributed}. Since $H_{1/2}^{-p} = H_{1}^{-p/2}$, searching over $p$ in $H_{1/2}^{-p}$ yields the same search space as $H_{1}^{-p}$. Therefore, the difference between $H_{1}$ and $H_{1/2}$ does not manifest practically in optimization speed, but it yields a significant difference in our understanding of how Shampoo works.

\section{Hessian Approximation of Shampoo} \label{sec:hess_shampoo}

From the Hessian approximation viewpoint, the previous section covers the case of using batch size $1$ and sampled labels, as described in Section \ref{sec:tech_hess}. To be precise, in Figure \ref{fig:main} top, we consider how well  $H_{\text{GN}}$ is correlated with $E_{x, s}[G_{x,s}G_{x,s}^T] \otimes E_{x, s}[G_{x,s}^TG_{x,s}]$, where $s$ represents that the labels are sampled from the model's output distribution. On the other hand, as discussed in Section~\ref{sec:tech_hess}, Shampoo in practice is generally used with arbitrary batch sizes and real labels. We now investigate the effect of these two factors on the Hessian approximation.

\subsection{Averaging gradients across the batch}
\label{sec:avg}

The next approximation towards Shampoo is to average the gradient across the batch, i.e., we go from 
\[ 
    L \leftarrow \E_{x, s \sim f(x)}[ G_{x, s} G_{x, s}^\top ] ;\quad  
    R \leftarrow \E_{x, s \sim f(x)}[ G_{x, s}^\top G_{x, s} ]
\]
to 
\[
    L \leftarrow |B| \E_{B,\bf{s}} [G_{B,\bf{s}}  G_{B,\bf{s}}^\top] ; \quad R \leftarrow |B| \E_{B,\bf{s}} [G_{B,\bf{s}}^\top  G_{B,\bf{s}}], 
\]
where $B$ denotes the batch, $\bf{s}$ is the concatenation of $s\sim f(x)$ for all $x \in B$ and $G_{B, \bf{s}} = \frac{1}{|B|} \sum_{x \in B, s=\mathbf{s} [x]} G_{x,s}$ is the batch gradient, with $\mathbf{s} [x]$ representing the sampled label corresponding to $x \in B$.

As previous works have shown, this change does not have any effect in expectation due to $G_{x, s}$ being mean zero for all $x$ when we take expectation over $s \sim f(x)$ \citep{f4a2a236-5292-37c9-8225-a973fbbd48c0} i.e. $\E_s[G_{x, s}] = 0$.
\begin{restatable}[Implicitly in \citet{liu2024sophia, osawa2023asdl}]{lemma}{lemavggrad}
    $$|B| \E_{B,\bf{s}} [G_{B,\bf{s}}  G_{B,\bf{s}}^\top] = \E_{x, s \sim f(x)}[ G_{x, s} G_{x, s}^\top ].$$
\end{restatable}
However, this does lead to a significant improvement in computational complexity by saving up to a factor of batch size. %
\subsection{Using real labels instead of sampled labels}
\label{sec:labels}

\begin{figure}[htbp]
\centering
\begin{tabular}{cccc}
       & \textnormal{\small {MNIST-2}} & \textnormal{\small {CIFAR-5M}} & \textnormal{\small {ImageNet}} \\
\rotatebox[origin=c]{90}{\textnormal{\small {Batch Size 1}}} & 
\animage{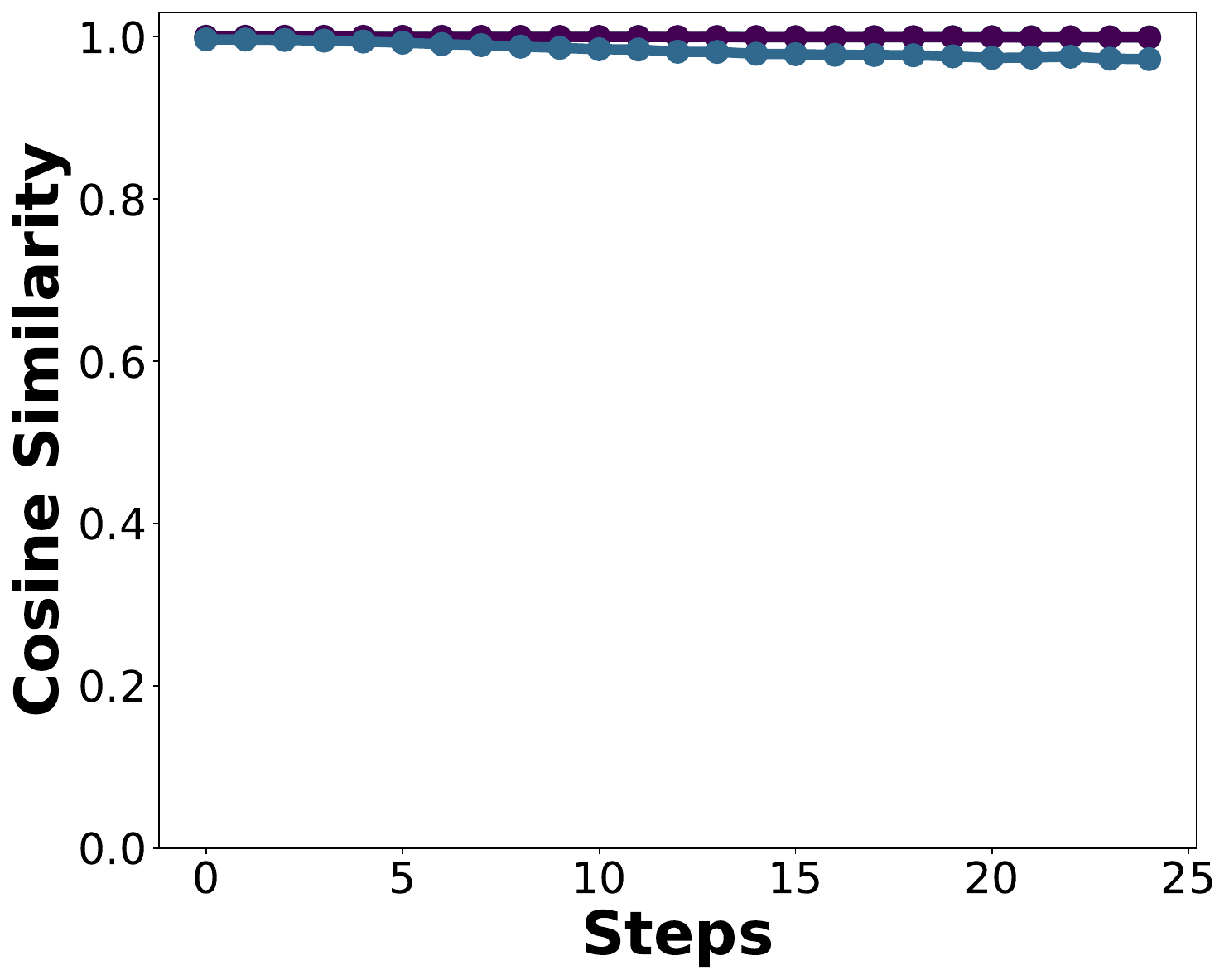} & 
\animage{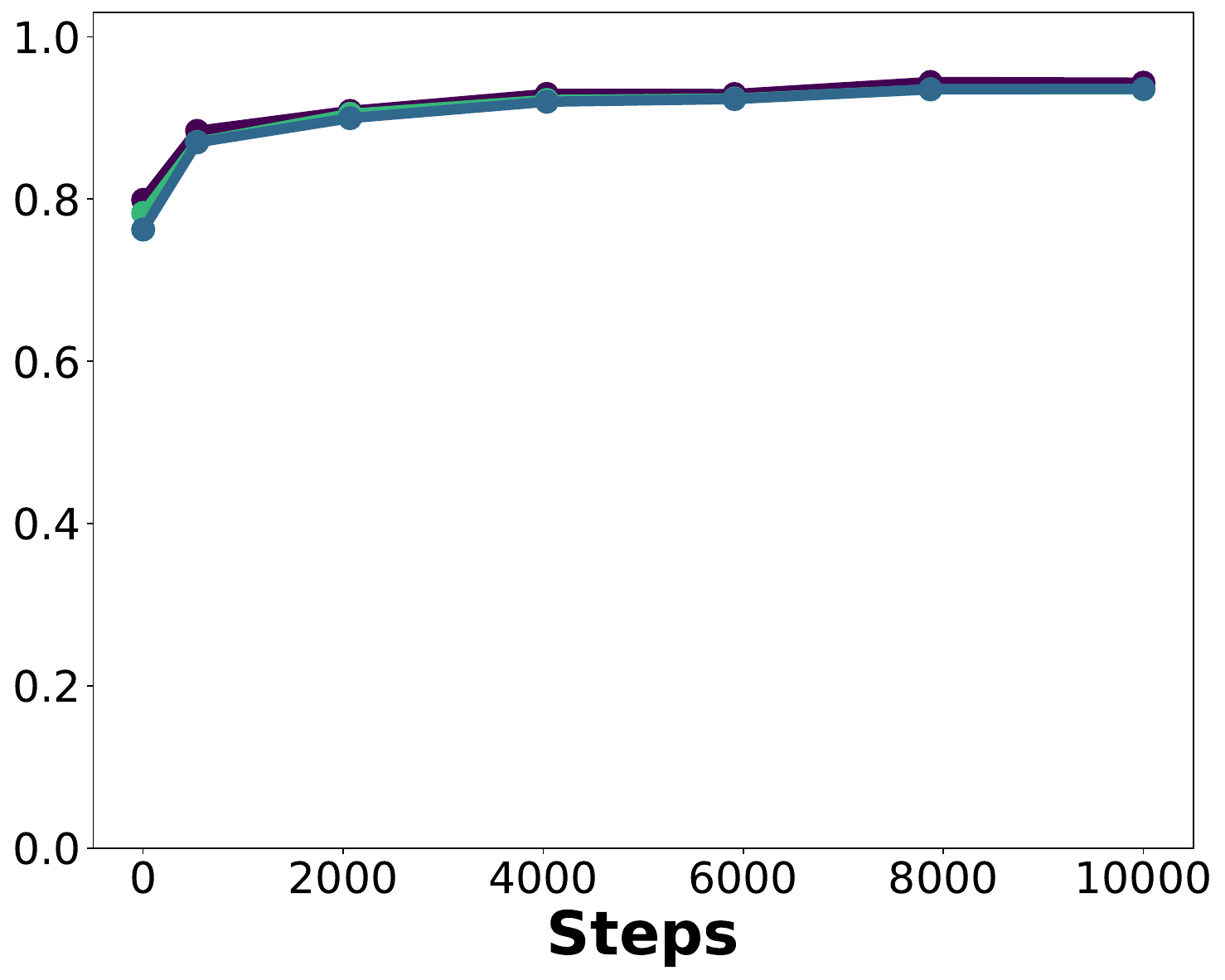} & 
\animage{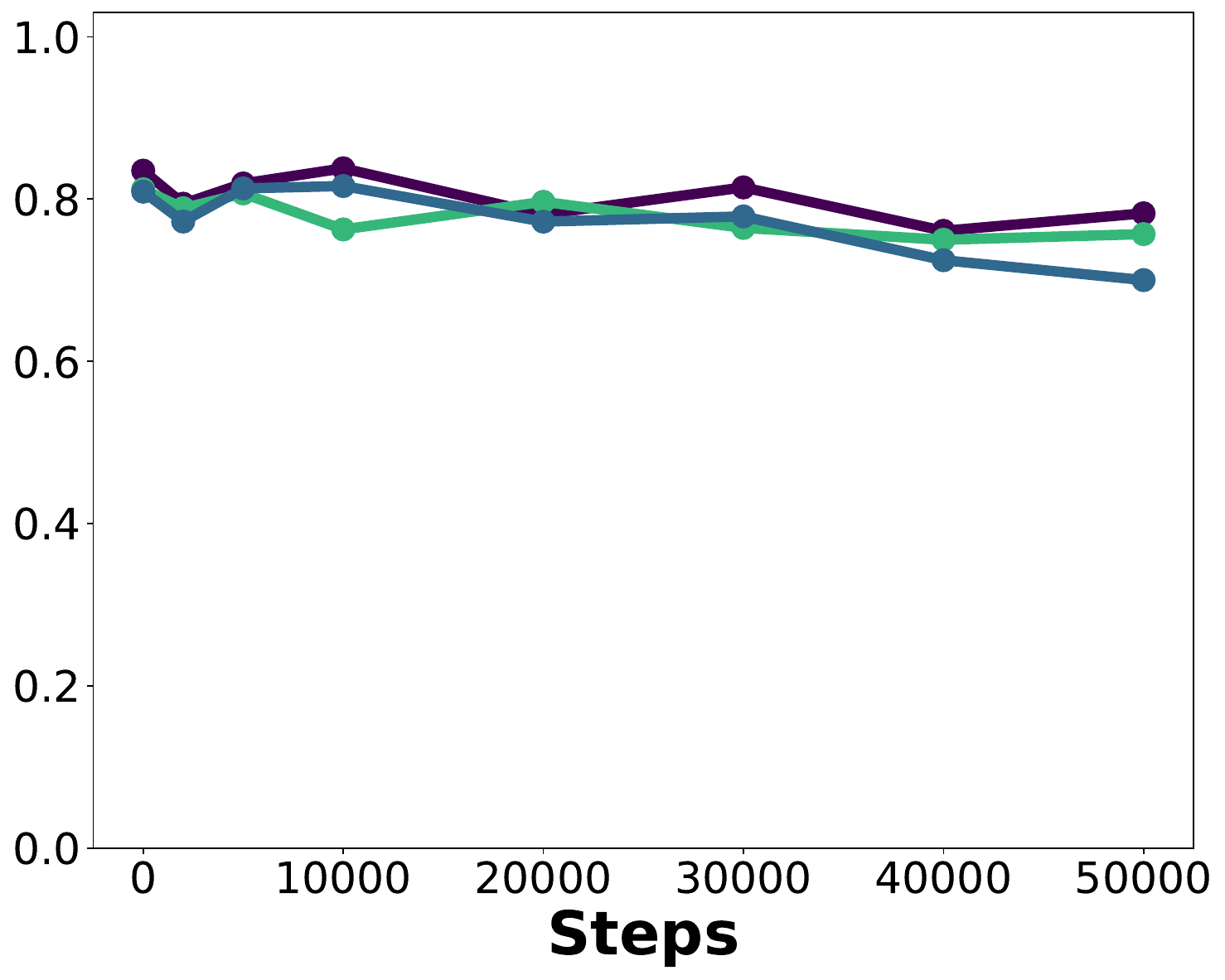} \\
\rotatebox[origin=c]{90}{\textnormal{\small {Batch Size 256}}} & 
\animage{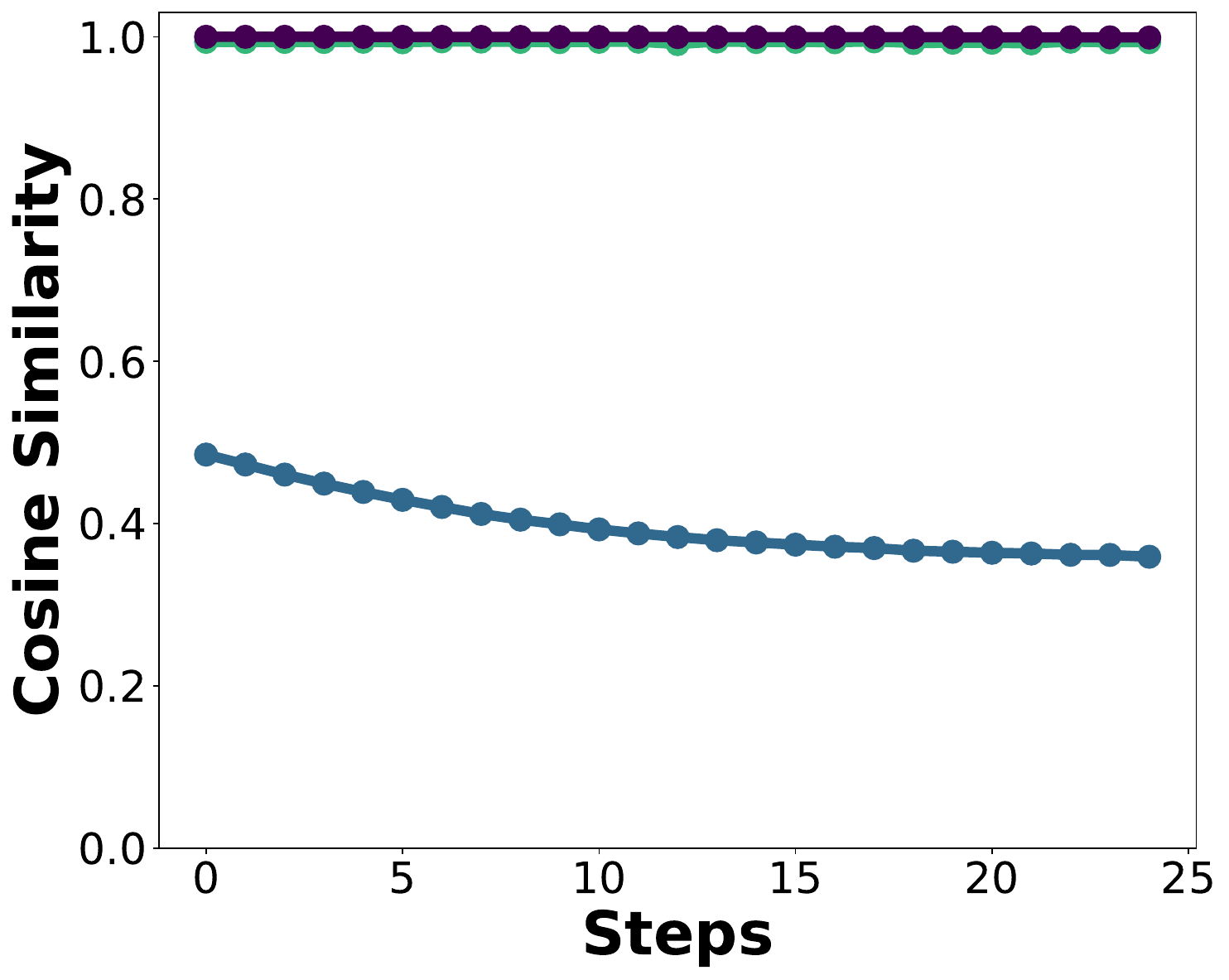} & 
\animage{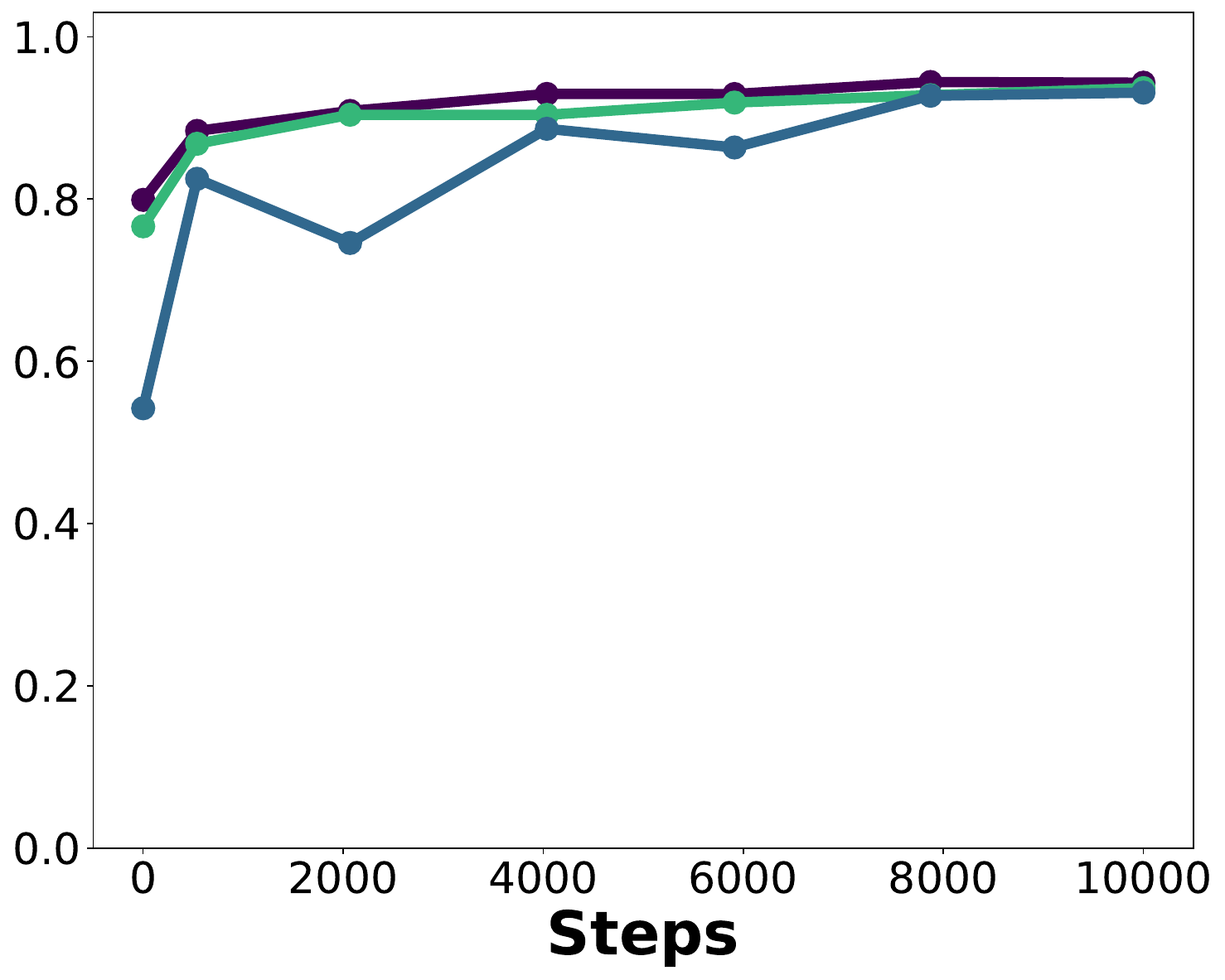} & 
\animage{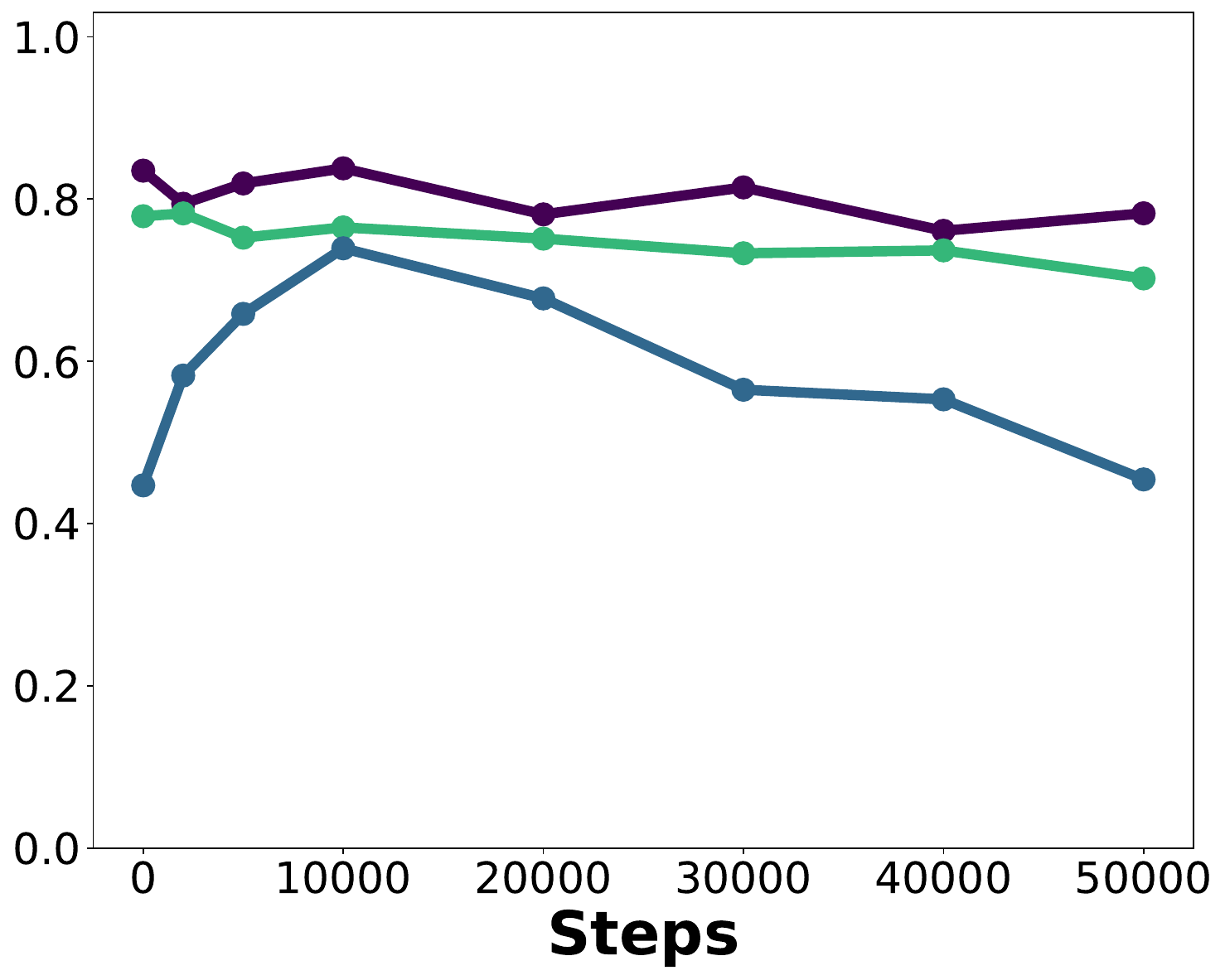} \\
\end{tabular}

\begin{minipage}{\textwidth}
    \centering
    \includegraphics[width=0.95\textwidth]{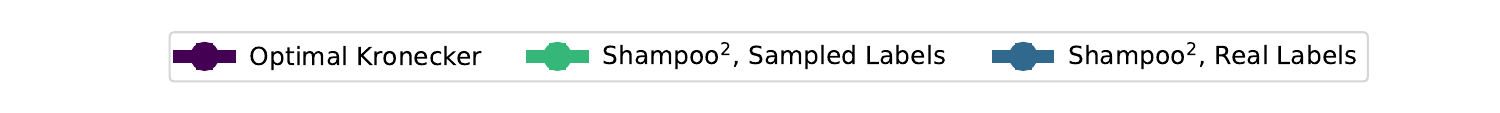}
\end{minipage}

    \caption{Cosine similarity between approximations of $H_{\text{GN}}$ and its true value. First row is for batch size 1 while the second row is for batch size 256. We observe deterioration in approximation quality at larger batch size. We note that the batch size does not refer to the batch size used in optimization, rather it refers to the batch size used for Hessian approximation.}
    \label{fig:figure4}

\end{figure}

As our final approximation we replace using sampled labels $s \sim f(x)$ to using real labels $y$. This approximation, denoted in the literature by empirical Fisher when batch size is 1, has been discussed at length by prior works~\citep{asdl,limitations_ef}. The main theoretical argument for why this approximation may work well is that, as we move towards optima, the two quantities converge in the presence of label noise \citep{rogergrnntd}.

In Figure~\ref{fig:figure4} (top), when evaluating $H_{\text{GN}}$ approximation with batch size $1$, we surprisingly find that the approximation quality is good throughout the training. However, unlike the case of sampled labels, the approximation starts to degrade at large batch sizes because the gradients with real labels are not mean 0. The lemma below \citep{rogergrnntd} shows how this estimator changes with batch size.

\begin{restatable}[\citet{rogergrnntd}]{lemma}{lemrealbatch}
    Let $B$ denote the batch and $G_{B} = \frac{1}{|B|} \sum_{(x,y) \in B} G_{x,y}$ denote the batch gradient. Then 
    $$\E_{B} [G_{B}  G_{B}^\top] = \frac{1}{|B|}\E_{x,y}[ G_{x, y} G_{x, y}^\top ] + \left(1 - \frac{1}{|B|}\right) \E_{x,y}[ G_{x, y}] \E_{x, y}[ G_{x, y}]^\top. $$
\end{restatable}

The above lemma shows that, depending on the batch size, the estimator interpolates between $\E_{x,y}[ G_{x, y} G_{x, y}^\top ]$ (Empirical Fisher) and $\E_{x,y}[ G_{x, y}] \E_{x, y}[ G_{x, y}]^\top$. As shown in Figure \ref{fig:figure4} (top), at batch size $1$, when $\E_{B} [G_{B}  G_{B}^\top ]$ is equal to $\E_{x,y}[ G_{x, y} G_{x, y}^\top ]$, it closely tracks the optimal Kronecker product approximation. In other words, approximating the empirical Fisher is nearly sufficient in our experiments to recover the optimal Kronecker product approximation to $H_{\text{GN}}$. However, with increasing batch size (Figure \ref{fig:figure4}, bottom row), the approximation quality degrades.

We note that this approximation has the computational benefit of not requiring another backpropagation with sampled labels; instead, these computations can be done alongside usual training.

\section{Related work} \label{sec:rel_work}
We discuss the related works in detail in Appendix~
\ref{app:rel_work}. Here, we discuss two closely related works: \citet{yi21} and \citet{koroko:hal-04266143}.

\citet{yi21} study the Hessian perspective of Shampoo and show that, under the assumption that sampled gradients follow a \textit{tensor-normal} distribution, the square of the Hessian estimate of Shampoo is perfectly correlated with $H_{\text{GN}}$. We also show the same result under much weaker conditions in Corollary \ref{corr:rank1}. Moreover, in Proposition~\ref{prop:main} we show that, in general, the square of the Hessian estimate of Shampoo is closely related to the optimal Kronecker product approximation of $H_{\text{GN}}$. We additionally also study the approximations used by Shampoo to make it computationally efficient (Section~\ref{sec:hess_shampoo}) and the Adagrad perspective of Shampoo's preconditioner.

\citet{van1993approximation} develop the theory of optimal Kronecker product approximation of a matrix (in Frobenius norm). \citet{koroko:hal-04266143} use it for finding layer-wise optimal Kronecker product approximation of $H_{\text{GN}}$ for a network without weight sharing. We extend their technique to networks with weight-sharing, and show that the square of the Hessian estimate of Shampoo is nearly equivalent to the optimal Kronecker product approximation of $H_{\text{GN}}$.

\section{Limitations} \label{sec:limitations}
The main contribution of our work is to show that the square of the Shampoo's approximation of $H$ (where $H$ refers to either $H_{\text{Ada}}$ or $H_{\text{GN}}$) is nearly equivalent to the optimal Kronecker approximation of $H$. Although we verify this empirically on various datasets and provide theoretical arguments, the gap between them depends on the problem structure. In some of our experiments with ViT architecture (Appendix~\ref{app:vit}), we find that the gap is relatively larger compared to other architectures. Moreover, it remains an open question to understand the conditions (beyond those described in K-FAC~\citet{martens2015optimizing}) under which $H$ is expected to be close to a Kronecker product. Again, in some of the experiments with ViTs (Appendix~\ref{app:vit}), we find that the optimal Kronecker product approximation to $H$ is much worse as compared to other architectures.

\section*{Acknowledgements}

NV and DM are supported by a Simons Investigator Fellowship, NSF grant DMS-2134157, DARPA grant W911NF2010021, and DOE grant DE-SC0022199. This work has been made possible in part by a gift from the Chan Zuckerberg Initiative Foundation to establish the Kempner Institute for the Study of Natural and Artificial Intelligence.
SK and DM acknowledge funding from the Office of Naval Research under award N00014-22-1-2377 and the National Science Foundation Grant under award \#IIS 2229881. LJ acknowledges funding from the National Science Foundation DMS-2134157.
\newpage
\bibliography{refs}

\begin{thebibliography}{47}
\providecommand{\natexlab}[1]{#1}
\providecommand{\url}[1]{\texttt{#1}}
\expandafter\ifx\csname urlstyle\endcsname\relax
  \providecommand{\doi}[1]{doi: #1}\else
  \providecommand{\doi}{doi: \begingroup \urlstyle{rm}\Url}\fi

\bibitem[Anil et~al.(2020)Anil, Gupta, Koren, Regan, and Singer]{anil2021towards}
Rohan Anil, Vineet Gupta, Tomer Koren, Kevin Regan, and Yoram Singer.
\newblock Towards practical second order optimization for deep learning.
\newblock 2020.

\bibitem[Anil et~al.(2021)Anil, Gupta, Koren, Regan, and Singer]{anil2021scalable}
Rohan Anil, Vineet Gupta, Tomer Koren, Kevin Regan, and Yoram Singer.
\newblock Scalable second order optimization for deep learning, 2021.

\bibitem[Anil et~al.(2022)Anil, Gadanho, Huang, Jacob, Li, Lin, Phillips, Pop, Regan, Shamir, et~al.]{anil2022factory}
Rohan Anil, Sandra Gadanho, Da~Huang, Nijith Jacob, Zhuoshu Li, Dong Lin, Todd Phillips, Cristina Pop, Kevin Regan, Gil~I Shamir, et~al.
\newblock On the factory floor: Ml engineering for industrial-scale ads recommendation models.
\newblock \emph{arXiv preprint arXiv:2209.05310}, 2022.

\bibitem[Balles et~al.(2020)Balles, Pedregosa, and Roux]{balles2020geometry}
Lukas Balles, Fabian Pedregosa, and Nicolas~Le Roux.
\newblock The geometry of sign gradient descent, 2020.

\bibitem[Bartlett(1953)]{f4a2a236-5292-37c9-8225-a973fbbd48c0}
M.~S. Bartlett.
\newblock Approximate confidence intervals.
\newblock \emph{Biometrika}, 40\penalty0 (1/2):\penalty0 12--19, 1953.
\newblock ISSN 00063444.

\bibitem[Cho et~al.(2015)Cho, Dhir, and Lee]{Minhyung15}
Minhyung Cho, Chandra Dhir, and Jaehyung Lee.
\newblock Hessian-free optimization for learning deep multidimensional recurrent neural networks.
\newblock In C.~Cortes, N.~Lawrence, D.~Lee, M.~Sugiyama, and R.~Garnett (eds.), \emph{Advances in Neural Information Processing Systems}, volume~28. Curran Associates, Inc., 2015.

\bibitem[Dahl et~al.(2023)Dahl, Schneider, Nado, Agarwal, Sastry, Hennig, Medapati, Eschenhagen, Kasimbeg, Suo, Bae, Gilmer, Peirson, Khan, Anil, Rabbat, Krishnan, Snider, Amid, Chen, Maddison, Vasudev, Badura, Garg, and Mattson]{dahl2023benchmarking}
George~E. Dahl, Frank Schneider, Zachary Nado, Naman Agarwal, Chandramouli~Shama Sastry, Philipp Hennig, Sourabh Medapati, Runa Eschenhagen, Priya Kasimbeg, Daniel Suo, Juhan Bae, Justin Gilmer, Abel~L. Peirson, Bilal Khan, Rohan Anil, Mike Rabbat, Shankar Krishnan, Daniel Snider, Ehsan Amid, Kongtao Chen, Chris~J. Maddison, Rakshith Vasudev, Michal Badura, Ankush Garg, and Peter Mattson.
\newblock Benchmarking neural network training algorithms, 2023.

\bibitem[Deng et~al.(2009)Deng, Dong, Socher, Li, Li, and Fei-Fei]{imagenet}
Jia Deng, Wei Dong, Richard Socher, Li-Jia Li, Kai Li, and Li~Fei-Fei.
\newblock Imagenet: A large-scale hierarchical image database.
\newblock In \emph{2009 IEEE Conference on Computer Vision and Pattern Recognition}, pp.\  248--255. Ieee, 2009.

\bibitem[Duchi et~al.(2011{\natexlab{a}})Duchi, Hazan, and Singer]{JMLR:v12:duchi11a}
John Duchi, Elad Hazan, and Yoram Singer.
\newblock Adaptive subgradient methods for online learning and stochastic optimization.
\newblock \emph{Journal of Machine Learning Research}, 12\penalty0 (61):\penalty0 2121--2159, 2011{\natexlab{a}}.

\bibitem[Duchi et~al.(2011{\natexlab{b}})Duchi, Hazan, and Singer]{duchi11}
John Duchi, Elad Hazan, and Yoram Singer.
\newblock Adaptive subgradient methods for online learning and stochastic optimization.
\newblock \emph{Journal of Machine Learning Research}, 12\penalty0 (61):\penalty0 2121--2159, 2011{\natexlab{b}}.

\bibitem[Duvvuri et~al.(2024)Duvvuri, Devvrit, Anil, Hsieh, and Dhillon]{duvvuri2024combining}
Sai~Surya Duvvuri, Fnu Devvrit, Rohan Anil, Cho-Jui Hsieh, and Inderjit~S Dhillon.
\newblock Combining axes preconditioners through kronecker approximation for deep learning.
\newblock In \emph{The Twelfth International Conference on Learning Representations}, 2024.

\bibitem[Eschenhagen et~al.(2023)Eschenhagen, Immer, Turner, Schneider, and Hennig]{eschenhagen2023kroneckerfactored}
Runa Eschenhagen, Alexander Immer, Richard~E Turner, Frank Schneider, and Philipp Hennig.
\newblock Kronecker-factored approximate curvature for modern neural network architectures.
\newblock In \emph{Thirty-seventh Conference on Neural Information Processing Systems}, 2023.

\bibitem[Gao et~al.(2020)Gao, Liu, Huang, Wang, Wang, Wang, Xu, and Yu]{gao2020eigenvaluecorrected}
Kai-Xin Gao, Xiao-Lei Liu, Zheng-Hai Huang, Min Wang, Shuangling Wang, Zidong Wang, Dachuan Xu, and Fan Yu.
\newblock Eigenvalue-corrected natural gradient based on a new approximation, 2020.

\bibitem[Gao et~al.(2021)Gao, Liu, Huang, Wang, Wang, Xu, and Yu]{Gao_Liu_Huang_Wang_Wang_Xu_Yu_2021}
Kaixin Gao, Xiaolei Liu, Zhenghai Huang, Min Wang, Zidong Wang, Dachuan Xu, and Fan Yu.
\newblock A trace-restricted kronecker-factored approximation to natural gradient.
\newblock \emph{Proceedings of the AAAI Conference on Artificial Intelligence}, 35\penalty0 (9):\penalty0 7519--7527, May 2021.
\newblock \doi{10.1609/aaai.v35i9.16921}.
\newblock URL \url{https://ojs.aaai.org/index.php/AAAI/article/view/16921}.

\bibitem[Garcia et~al.(2023)Garcia, Freddi, Fotiadis, Li, Vakili, Bernacchia, and Hennequin]{garcia2023fisherlegendre}
Jezabel~R Garcia, Federica Freddi, Stathi Fotiadis, Maolin Li, Sattar Vakili, Alberto Bernacchia, and Guillaume Hennequin.
\newblock Fisher-legendre (fishleg) optimization of deep neural networks.
\newblock In \emph{The Eleventh International Conference on Learning Representations}, 2023.
\newblock URL \url{https://openreview.net/forum?id=c9lAOPvQHS}.

\bibitem[Gemini~Team(20024)]{gemini15}
Google Gemini~Team.
\newblock {Gemini 1.5: Unlocking multimodal understanding across millions of tokens of context}.
\newblock \url{https://storage.googleapis.com/deepmind-media/gemini/gemini_v1_5_report.pdf}, 20024.
\newblock [Online; accessed 19-May-2024].

\bibitem[George et~al.(2018)George, Laurent, Bouthillier, Ballas, and Vincent]{Thomas18}
Thomas George, C\'{e}sar Laurent, Xavier Bouthillier, Nicolas Ballas, and Pascal Vincent.
\newblock Fast approximate natural gradient descent in a kronecker factored eigenbasis.
\newblock In S.~Bengio, H.~Wallach, H.~Larochelle, K.~Grauman, N.~Cesa-Bianchi, and R.~Garnett (eds.), \emph{Advances in Neural Information Processing Systems}, volume~31. Curran Associates, Inc., 2018.
\newblock URL \url{https://proceedings.neurips.cc/paper_files/paper/2018/file/48000647b315f6f00f913caa757a70b3-Paper.pdf}.

\bibitem[Golub \& Van~Loan(1996)Golub and Van~Loan]{GoluVanl96}
Gene~H. Golub and Charles~F. Van~Loan.
\newblock \emph{Matrix Computations}.
\newblock The Johns Hopkins University Press, third edition, 1996.

\bibitem[Grosse(2021)]{rogergrnntd}
Roger Grosse.
\newblock Adaptive gradient methods, normalization, and weight decay.
\newblock \url{https://www.cs.toronto.edu/~rgrosse/courses/csc2541_2021/readings/L05_normalization.pdf}, 2021.

\bibitem[Gupta et~al.(2018{\natexlab{a}})Gupta, Koren, and Singer]{gupta18}
Vineet Gupta, Tomer Koren, and Yoram Singer.
\newblock Shampoo: Preconditioned stochastic tensor optimization.
\newblock In Jennifer Dy and Andreas Krause (eds.), \emph{Proceedings of the 35th International Conference on Machine Learning}, volume~80 of \emph{Proceedings of Machine Learning Research}, pp.\  1842--1850. PMLR, 10--15 Jul 2018{\natexlab{a}}.
\newblock URL \url{https://proceedings.mlr.press/v80/gupta18a.html}.

\bibitem[Gupta et~al.(2018{\natexlab{b}})Gupta, Koren, and Singer]{gupta2018shampoo}
Vineet Gupta, Tomer Koren, and Yoram Singer.
\newblock Shampoo: Preconditioned stochastic tensor optimization.
\newblock In \emph{International Conference on Machine Learning}, pp.\  1842--1850. PMLR, 2018{\natexlab{b}}.

\bibitem[He et~al.(2016)He, Zhang, Ren, and Sun]{resnet}
Kaiming He, Xiangyu Zhang, Shaoqing Ren, and Jian Sun.
\newblock Deep residual learning for image recognition.
\newblock In \emph{Proceedings of the IEEE conference on computer vision and pattern recognition}, pp.\  770--778, 2016.

\bibitem[Henderson \& Searle(1981)Henderson and Searle]{kronecker}
Harold~V Henderson and Shayle~R Searle.
\newblock The vec-permutation matrix, the vec operator and kronecker products: A review.
\newblock \emph{Linear and multilinear algebra}, 9\penalty0 (4):\penalty0 271--288, 1981.

\bibitem[Koroko et~al.(2023{\natexlab{a}})Koroko, Anciaux-Sedrakian, Gharbia, Gar{\`e}s, Haddou, and Tran]{koroko:hal-04266143}
Abdoulaye Koroko, Ani Anciaux-Sedrakian, Ibtihel Gharbia, Val{\'e}rie Gar{\`e}s, Mounir Haddou, and Quang~Huy Tran.
\newblock {Efficient approximations of the fisher matrix in neural networks using kronecker product singular value decomposition}.
\newblock \emph{{ESAIM: Proceedings and Surveys}}, 73:\penalty0 218--237, 2023{\natexlab{a}}.
\newblock \doi{10.1051/proc/202373218}.
\newblock URL \url{https://hal.science/hal-04266143}.

\bibitem[Koroko et~al.(2023{\natexlab{b}})Koroko, Anciaux-Sedrakian, Gharbia, Gar{\`e}s, Haddou, and Tran]{koroko2023efficient}
Abdoulaye Koroko, Ani Anciaux-Sedrakian, Ibtihel~Ben Gharbia, Val{\'e}rie Gar{\`e}s, Mounir Haddou, and Quang~Huy Tran.
\newblock Efficient approximations of the fisher matrix in neural networks using kronecker product singular value decomposition.
\newblock \emph{ESAIM: Proceedings and Surveys}, 73:\penalty0 218--237, 2023{\natexlab{b}}.

\bibitem[Kunstner et~al.(2019)Kunstner, Hennig, and Balles]{limitations_ef}
Frederik Kunstner, Philipp Hennig, and Lukas Balles.
\newblock Limitations of the empirical fisher approximation for natural gradient descent.
\newblock In H.~Wallach, H.~Larochelle, A.~Beygelzimer, F.~d\textquotesingle Alch\'{e}-Buc, E.~Fox, and R.~Garnett (eds.), \emph{Advances in Neural Information Processing Systems}, volume~32. Curran Associates, Inc., 2019.
\newblock URL \url{https://proceedings.neurips.cc/paper_files/paper/2019/file/46a558d97954d0692411c861cf78ef79-Paper.pdf}.

\bibitem[LeCun et~al.(1998)LeCun, Bottou, Bengio, and Haffner]{mnist}
Yann LeCun, Léon Bottou, Yoshua Bengio, and Patrick Haffner.
\newblock Gradient-based learning applied to document recognition.
\newblock \emph{Proceedings of the IEEE}, 86\penalty0 (11):\penalty0 2278--2324, 1998.

\bibitem[Lin et~al.(2024)Lin, Dangel, Eschenhagen, Bae, Turner, and Makhzani]{lin2024remove}
Wu~Lin, Felix Dangel, Runa Eschenhagen, Juhan Bae, Richard~E. Turner, and Alireza Makhzani.
\newblock Can we remove the square-root in adaptive gradient methods? a second-order perspective.
\newblock ar{X}iv 2402.03496, 2024.

\bibitem[Liu et~al.(2024)Liu, Li, Hall, Liang, and Ma]{liu2024sophia}
Hong Liu, Zhiyuan Li, David Leo~Wright Hall, Percy Liang, and Tengyu Ma.
\newblock Sophia: A scalable stochastic second-order optimizer for language model pre-training.
\newblock In \emph{The Twelfth International Conference on Learning Representations}, 2024.

\bibitem[Liu et~al.(2022)Liu, Mao, Wu, Feichtenhofer, Darrell, and Xie]{Liu_2022_CVPR}
Zhuang Liu, Hanzi Mao, Chao-Yuan Wu, Christoph Feichtenhofer, Trevor Darrell, and Saining Xie.
\newblock A convnet for the 2020s.
\newblock In \emph{Proceedings of the IEEE/CVF Conference on Computer Vision and Pattern Recognition (CVPR)}, pp.\  11976--11986, June 2022.

\bibitem[Loan \& Pitsianis(1993)Loan and Pitsianis]{van1993approximation}
C.~F.~Van Loan and N.~Pitsianis.
\newblock Approximation with kronecker products.
\newblock In Bart L. R.~Moor Marc S.~Moonen, Gene H.~Golub (ed.), \emph{Linear Algebra for Large Scale and Real-Time Applications}, pp.\  293--314. Springer, 1993.

\bibitem[Martens(2010)]{martens10}
James Martens.
\newblock Deep learning via hessian-free optimization.
\newblock In Johannes F{\"{u}}rnkranz and Thorsten Joachims (eds.), \emph{Proceedings of the 27th International Conference on Machine Learning (ICML-10), June 21-24, 2010, Haifa, Israel}, pp.\  735--742. Omnipress, 2010.
\newblock URL \url{https://icml.cc/Conferences/2010/papers/458.pdf}.

\bibitem[Martens \& Grosse(2015{\natexlab{a}})Martens and Grosse]{martens15}
James Martens and Roger Grosse.
\newblock Optimizing neural networks with kronecker-factored approximate curvature.
\newblock In Francis Bach and David Blei (eds.), \emph{Proceedings of the 32nd International Conference on Machine Learning}, volume~37 of \emph{Proceedings of Machine Learning Research}, pp.\  2408--2417, Lille, France, 07--09 Jul 2015{\natexlab{a}}. PMLR.
\newblock URL \url{https://proceedings.mlr.press/v37/martens15.html}.

\bibitem[Martens \& Grosse(2015{\natexlab{b}})Martens and Grosse]{martens2015optimizing}
James Martens and Roger Grosse.
\newblock Optimizing neural networks with kronecker-factored approximate curvature.
\newblock In \emph{International conference on machine learning}, pp.\  2408--2417. PMLR, 2015{\natexlab{b}}.

\bibitem[Martens \& Sutskever(2011)Martens and Sutskever]{martens11}
James Martens and Ilya Sutskever.
\newblock Learning recurrent neural networks with hessian-free optimization.
\newblock In \emph{Proceedings of the 28th International Conference on International Conference on Machine Learning}, ICML'11, pp.\  1033–1040, Madison, WI, USA, 2011. Omnipress.
\newblock ISBN 9781450306195.

\bibitem[Martens et~al.(2018)Martens, Ba, and Johnson]{martens2018kroneckerrecurrent}
James Martens, Jimmy Ba, and Matt Johnson.
\newblock Kronecker-factored curvature approximations for recurrent neural networks.
\newblock In \emph{International Conference on Learning Representations}, 2018.
\newblock URL \url{https://openreview.net/forum?id=HyMTkQZAb}.

\bibitem[Martens et~al.(2010)]{martens2010deep}
James Martens et~al.
\newblock Deep learning via hessian-free optimization.
\newblock In \emph{Icml}, volume~27, pp.\  735--742, 2010.

\bibitem[Nakkiran et~al.(2020)Nakkiran, Neyshabur, and Sedghi]{cifar5m}
Preetum Nakkiran, Behnam Neyshabur, and Hanie Sedghi.
\newblock The deep bootstrap framework: Good online learners are good offline generalizers.
\newblock \emph{arXiv preprint arXiv:2010.08127}, 2020.

\bibitem[Osawa et~al.(2019)Osawa, Tsuji, Ueno, Naruse, Yokota, and Matsuoka]{kazuki19}
Kazuki Osawa, Yohei Tsuji, Yuichiro Ueno, Akira Naruse, Rio Yokota, and Satoshi Matsuoka.
\newblock Large-scale distributed second-order optimization using kronecker-factored approximate curvature for deep convolutional neural networks.
\newblock In \emph{2019 IEEE/CVF Conference on Computer Vision and Pattern Recognition (CVPR)}, pp.\  12351--12359, 2019.
\newblock \doi{10.1109/CVPR.2019.01264}.

\bibitem[Osawa et~al.(2023{\natexlab{a}})Osawa, Ishikawa, Yokota, Li, and Hoefler]{asdl}
Kazuki Osawa, Satoki Ishikawa, Rio Yokota, Shigang Li, and Torsten Hoefler.
\newblock {ASDL:} {A} unified interface for gradient preconditioning in pytorch.
\newblock \emph{CoRR}, abs/2305.04684, 2023{\natexlab{a}}.
\newblock \doi{10.48550/ARXIV.2305.04684}.
\newblock URL \url{https://doi.org/10.48550/arXiv.2305.04684}.

\bibitem[Osawa et~al.(2023{\natexlab{b}})Osawa, Ishikawa, Yokota, Li, and Hoefler]{osawa2023asdl}
Kazuki Osawa, Satoki Ishikawa, Rio Yokota, Shigang Li, and Torsten Hoefler.
\newblock Asdl: A unified interface for gradient preconditioning in pytorch, 2023{\natexlab{b}}.

\bibitem[Papyan(2019)]{papyan19}
Vardan Papyan.
\newblock Measurements of three-level hierarchical structure in the outliers in the spectrum of deepnet hessians.
\newblock In Kamalika Chaudhuri and Ruslan Salakhutdinov (eds.), \emph{Proceedings of the 36th International Conference on Machine Learning}, volume~97 of \emph{Proceedings of Machine Learning Research}, pp.\  5012--5021. PMLR, 09--15 Jun 2019.
\newblock URL \url{https://proceedings.mlr.press/v97/papyan19a.html}.

\bibitem[Pascanu \& Bengio(2014)Pascanu and Bengio]{DBLP:journals/corr/abs-1301-3584}
Razvan Pascanu and Yoshua Bengio.
\newblock Revisiting natural gradient for deep networks.
\newblock In Yoshua Bengio and Yann LeCun (eds.), \emph{2nd International Conference on Learning Representations, {ICLR} 2014, Banff, AB, Canada, April 14-16, 2014, Conference Track Proceedings}, 2014.
\newblock URL \url{http://arxiv.org/abs/1301.3584}.

\bibitem[Ren \& Goldfarb(2021)Ren and Goldfarb]{yi21}
Yi~Ren and Donald Goldfarb.
\newblock Tensor normal training for deep learning models.
\newblock In M.~Ranzato, A.~Beygelzimer, Y.~Dauphin, P.S. Liang, and J.~Wortman Vaughan (eds.), \emph{Advances in Neural Information Processing Systems}, volume~34, pp.\  26040--26052. Curran Associates, Inc., 2021.
\newblock URL \url{https://proceedings.neurips.cc/paper_files/paper/2021/file/dae3312c4c6c7000a37ecfb7b0aeb0e4-Paper.pdf}.

\bibitem[Sankar et~al.(2021)Sankar, Khasbage, Vigneswaran, and N~Balasubramanian]{Sankar_Khasbage_Vigneswaran_Balasubramanian_2021}
Adepu~Ravi Sankar, Yash Khasbage, Rahul Vigneswaran, and Vineeth N~Balasubramanian.
\newblock A deeper look at the hessian eigenspectrum of deep neural networks and its applications to regularization.
\newblock \emph{Proceedings of the AAAI Conference on Artificial Intelligence}, 35\penalty0 (11):\penalty0 9481--9488, May 2021.
\newblock \doi{10.1609/aaai.v35i11.17142}.
\newblock URL \url{https://ojs.aaai.org/index.php/AAAI/article/view/17142}.

\bibitem[Shi et~al.(2023)Shi, Lee, Iwasaki, Gallego-Posada, Li, Rangadurai, Mudigere, and Rabbat]{shi2023distributed}
Hao-Jun~Michael Shi, Tsung-Hsien Lee, Shintaro Iwasaki, Jose Gallego-Posada, Zhijing Li, Kaushik Rangadurai, Dheevatsa Mudigere, and Michael Rabbat.
\newblock A distributed data-parallel pytorch implementation of the distributed shampoo optimizer for training neural networks at-scale, 2023.

\bibitem[Van~Loan \& Pitsianis(1993)Van~Loan and Pitsianis]{approximation_with_kronecker}
Charles~F Van~Loan and Nikos Pitsianis.
\newblock \emph{Approximation with Kronecker products}.
\newblock Springer, 1993.

\end{thebibliography}
\bibliographystyle{iclr2024_conference}

\newpage
\appendix
\section{Additional experimental results} \label{app:vit}

\begin{figure}[htbp]
\centering

\begin{tabular}{cccc}
       &     \textnormal{\small {{MNIST-2}}} & \textnormal{\small {{CIFAR-5M (ResNet)}}} & \textnormal{\small {{ImageNet}}} \\
\rotatebox[origin=c]{90}{\textnormal{\small {Gauss--Newton}}} & 
\animage{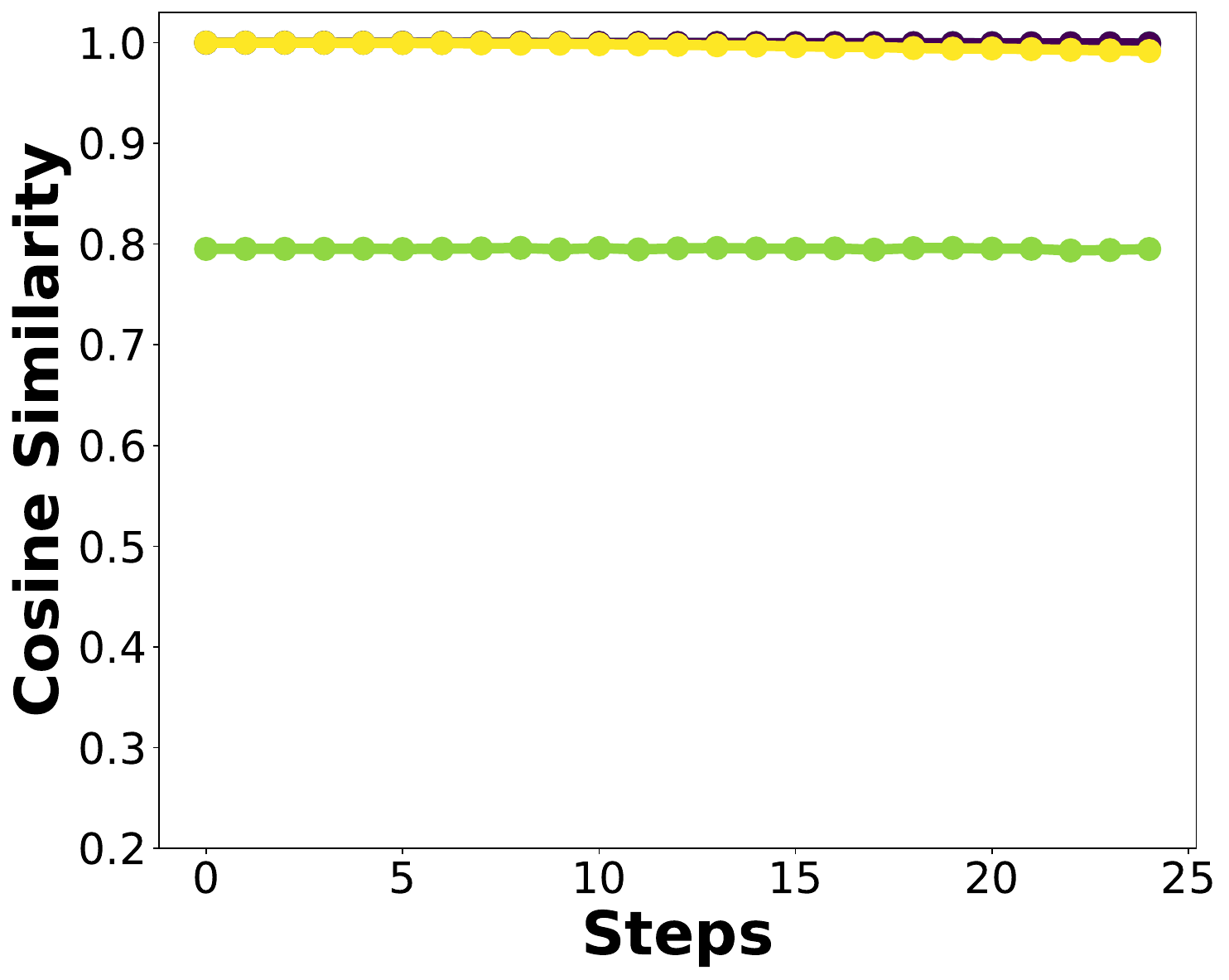} & 
\animage{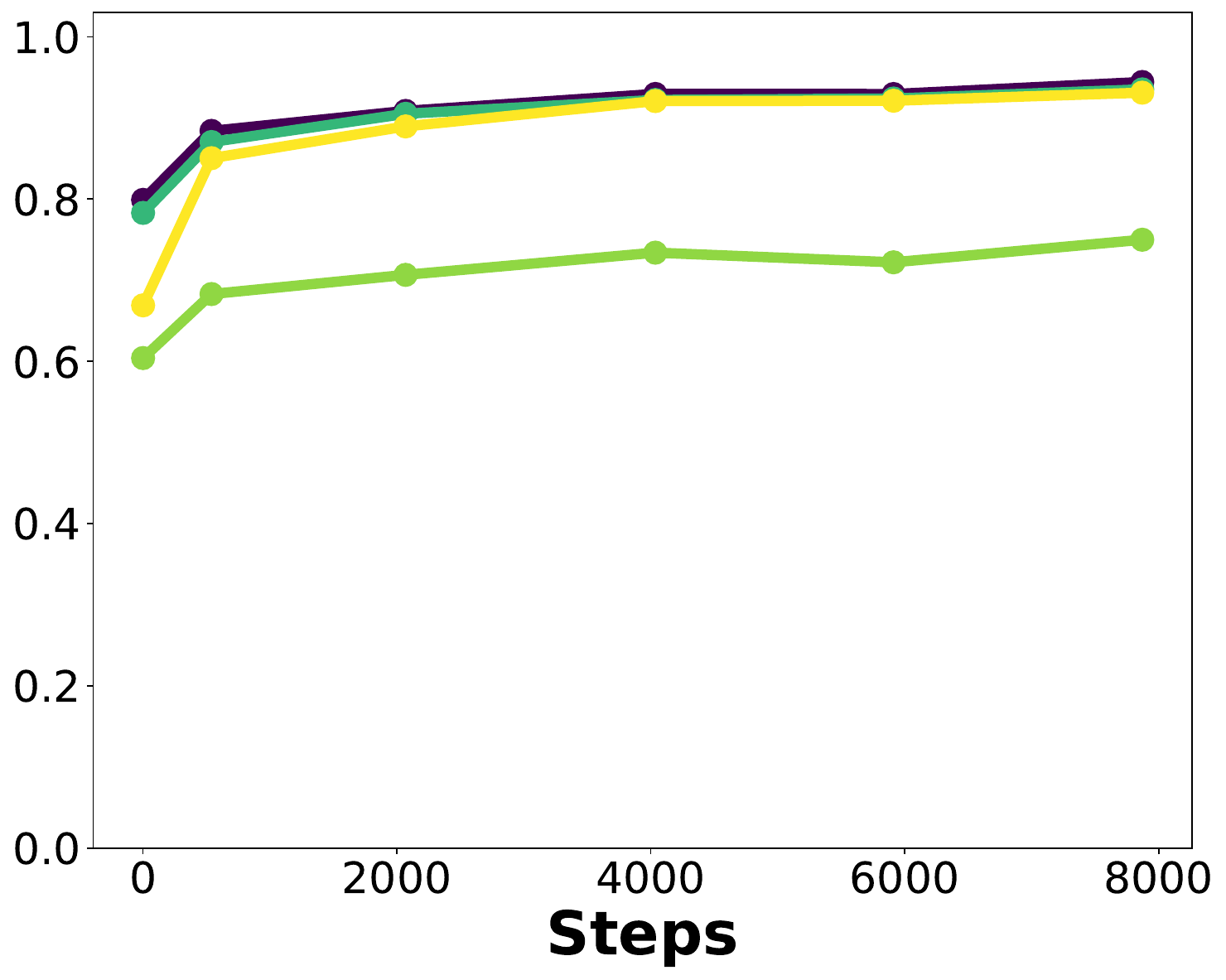} & 
\animage{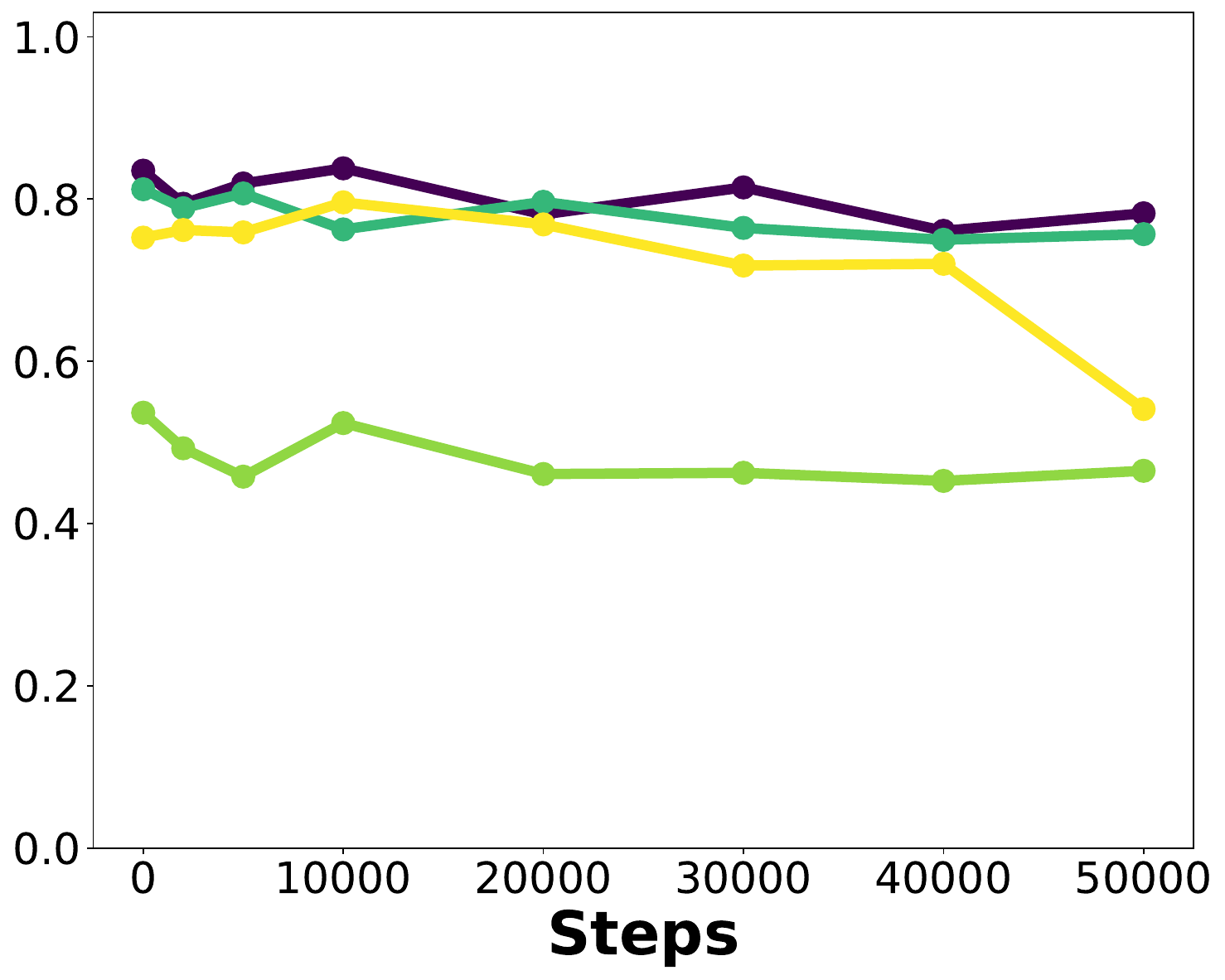} \\
\end{tabular}

\begin{subfigure}[b]{\textwidth}
    \includegraphics[width=0.95\textwidth]{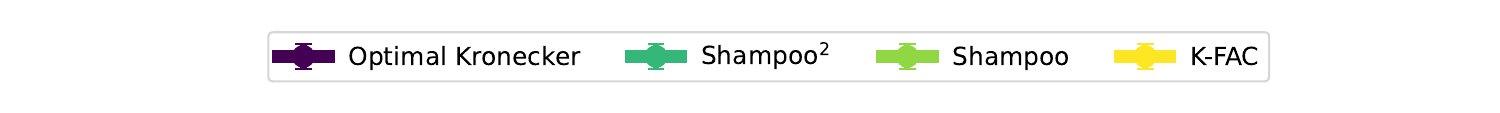}
\end{subfigure}
\vspace{-2em}
\caption{Cosine similarity between different approximations of the Gauss--Newton (GN) component of the Hessian and its true value for different datasets and architectures. As can be seen, $\text{Shampoo}^2$ tracks the optimal Kronecker approximation much more closely than Shampoo. These plots also include the K-FAC approximation, and we note that $\text{Shampoo}^2$ always outperforms K-FAC, though they are close in some settings.}
\label{fig:main_full_appendix}
\end{figure}

\subsection{ViT architecture}
\begin{figure}[htbp]
\centering
\begin{tabular}{cccc}
       &     \textnormal{\small {{FFN Linear Layer 1}}} & \textnormal{\small {{FFN Linear Layer 2}}} & \textnormal{\small {{Q-K Projection Layer}}} \\
\rotatebox[origin=c]{90}{\textnormal{\small {Gauss--Newton}}} & 
\animage{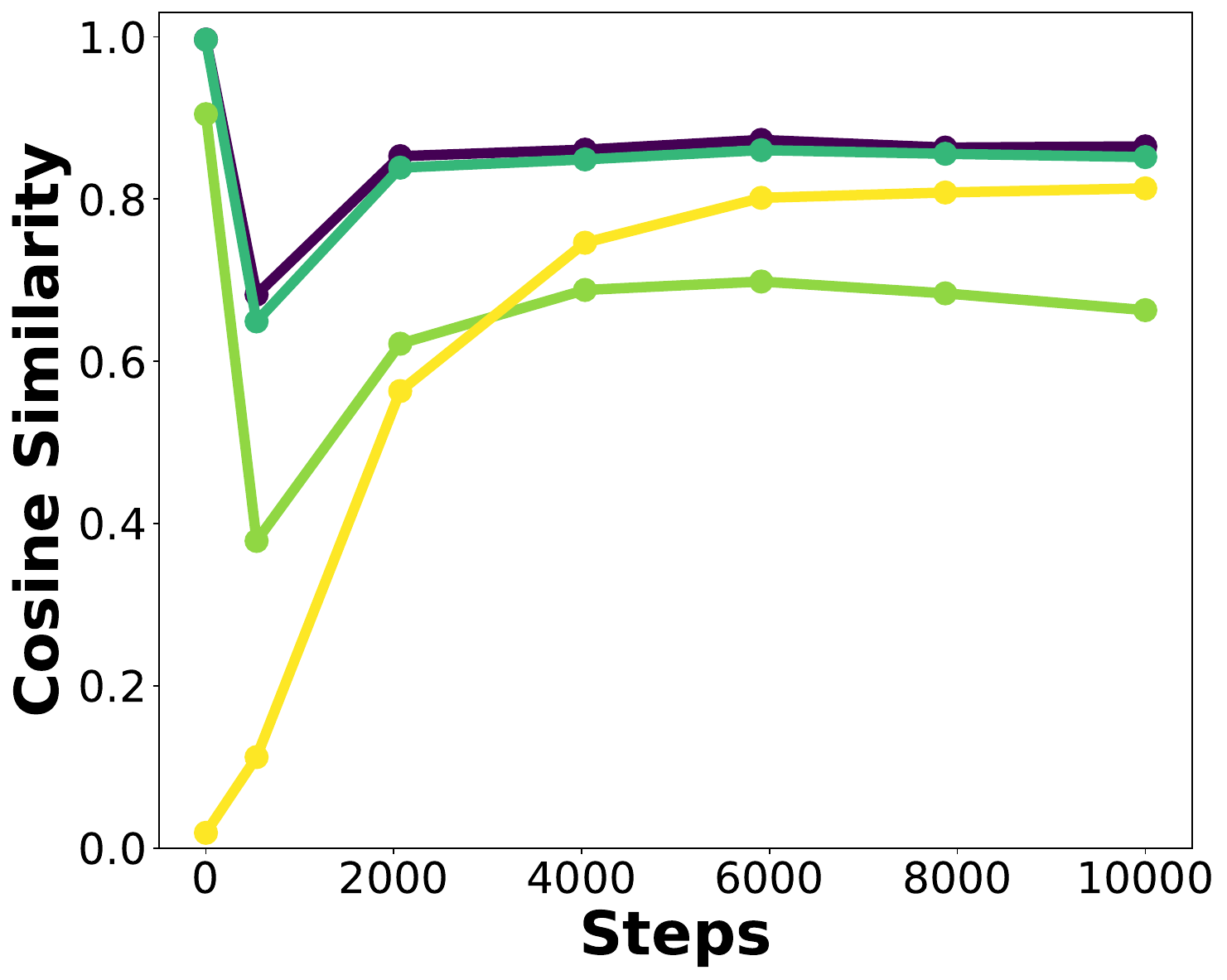} & 
\animage{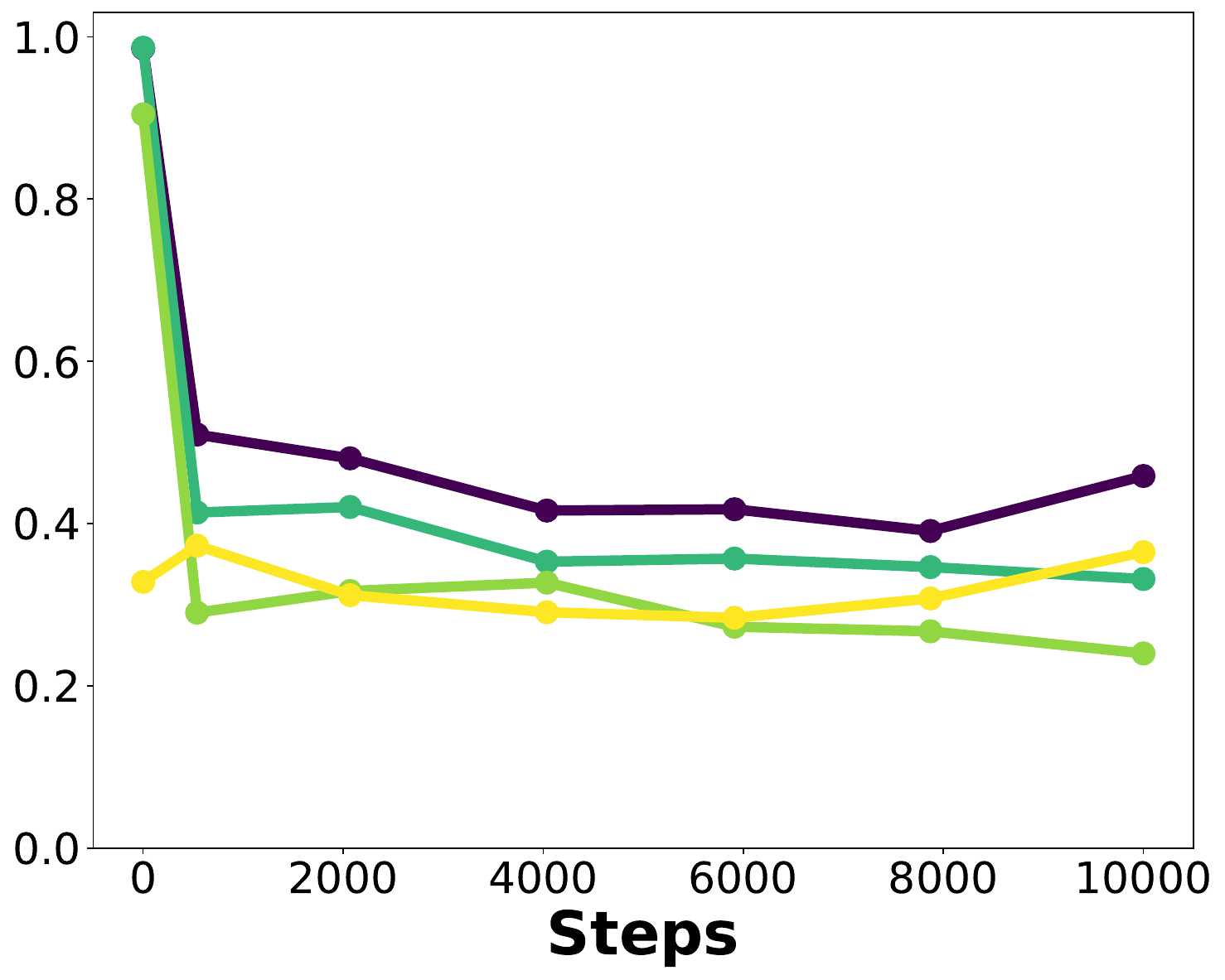} & 
\animage{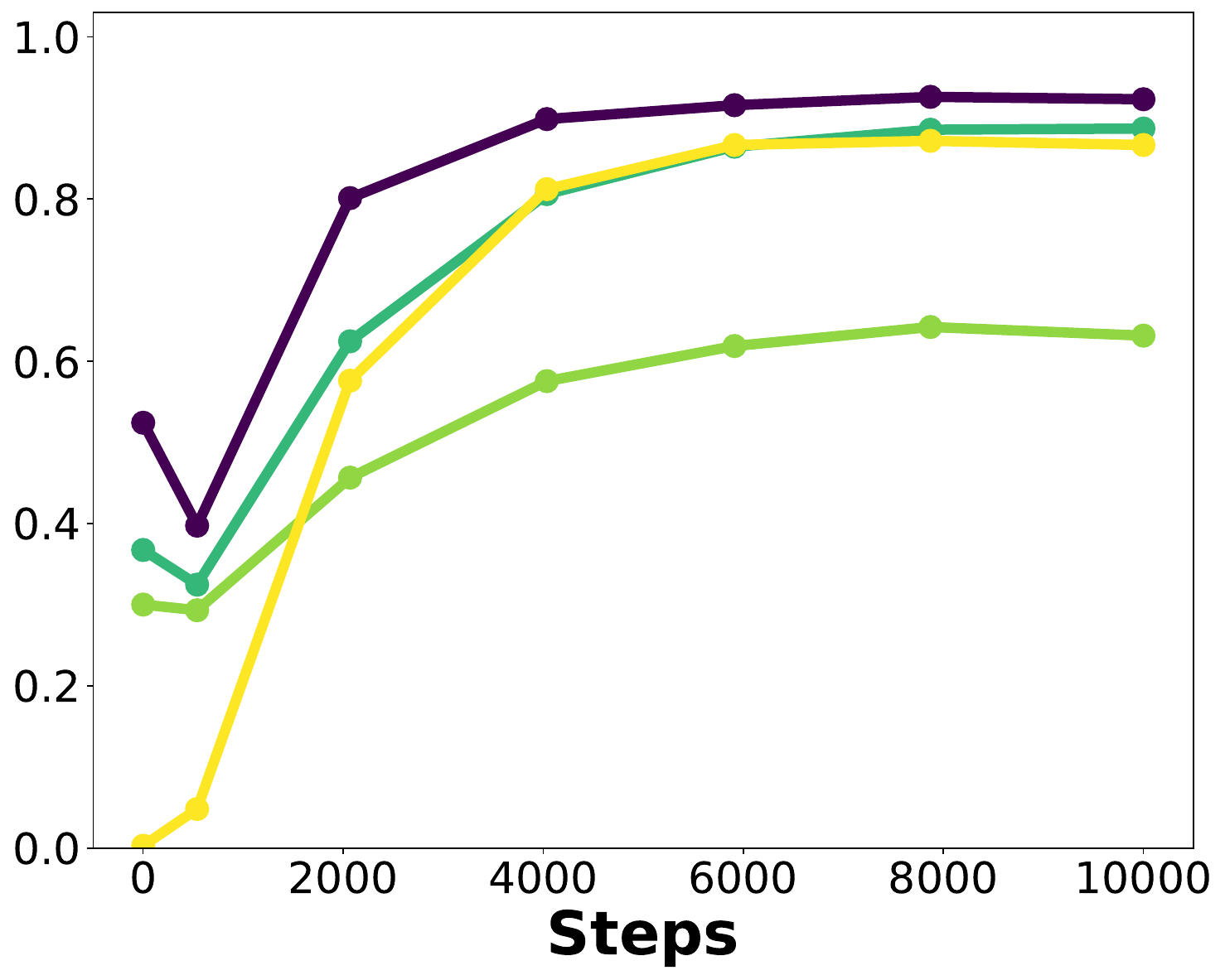} \\
\rotatebox[origin=c]{90}{\textnormal{\small {Adagrad}}} & 
\animage{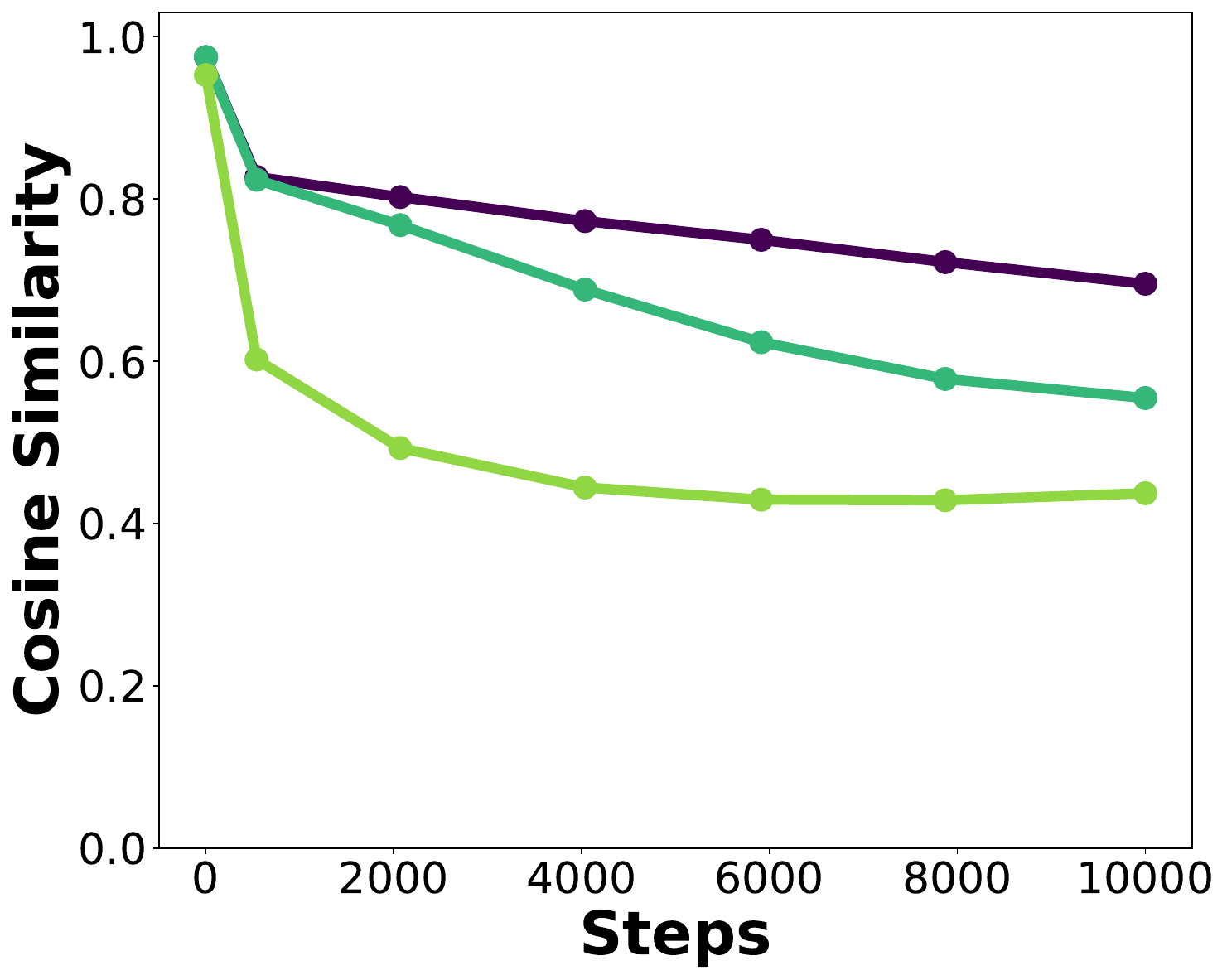} & 
\animage{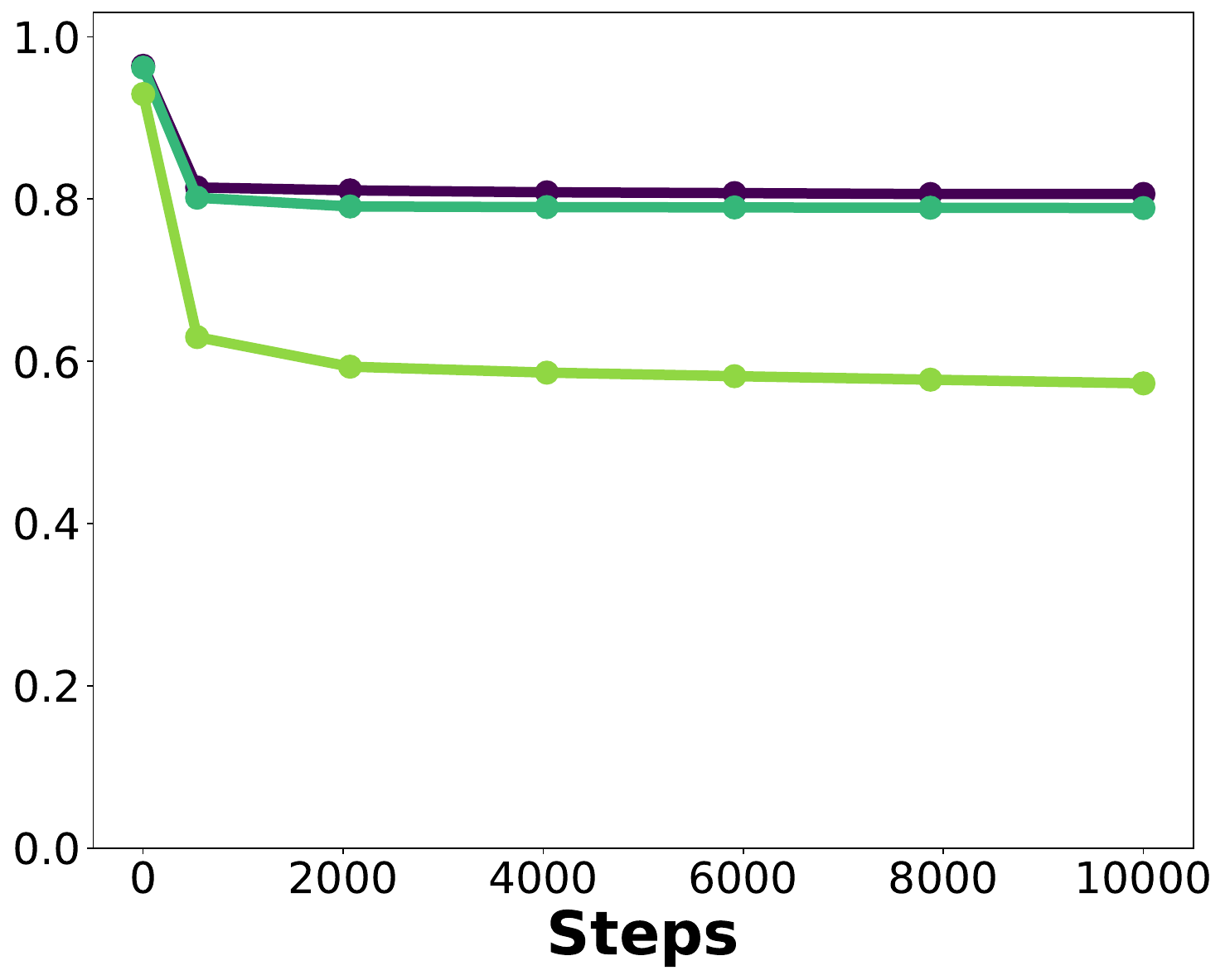} & 
\animage{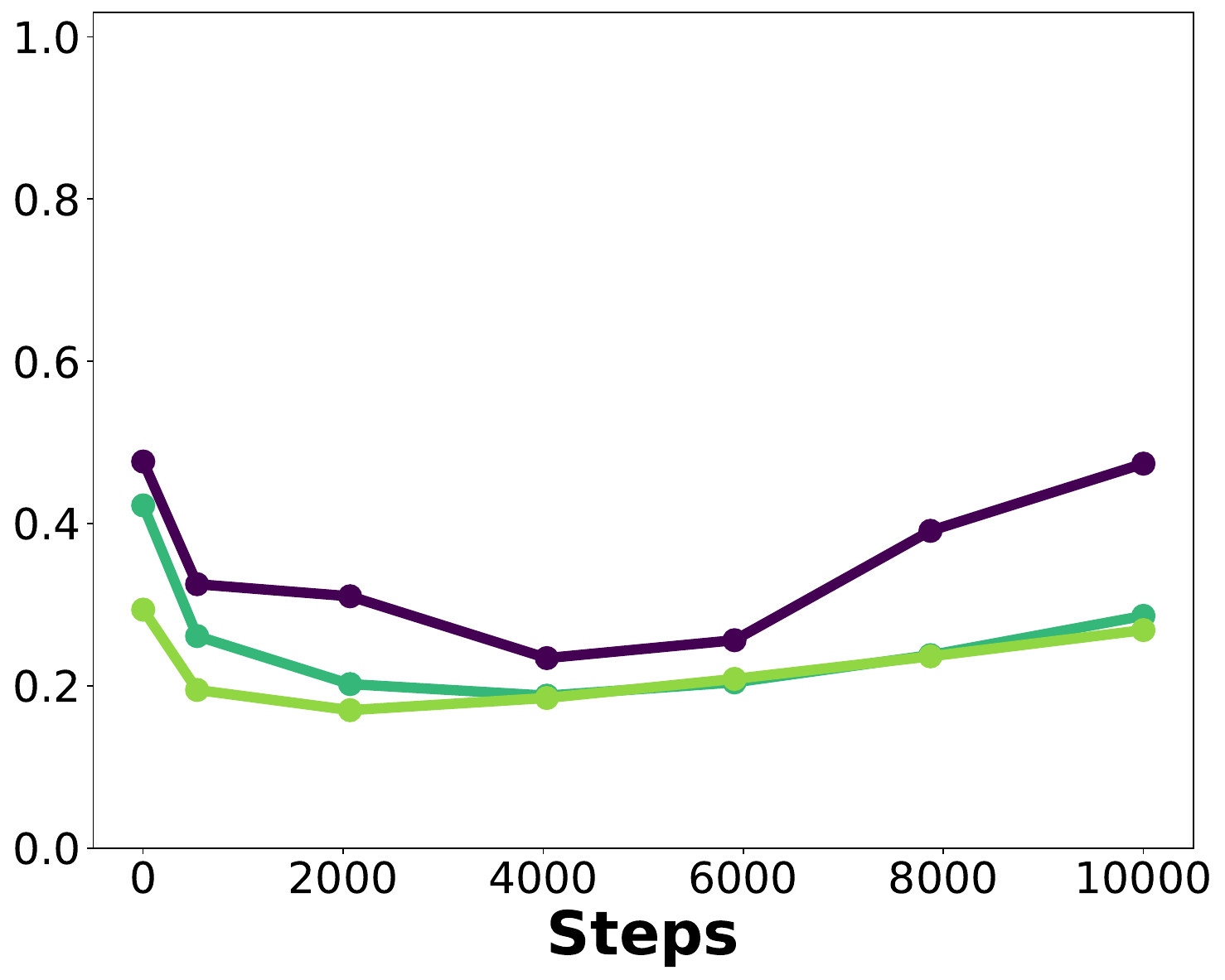} \\
\end{tabular}
\begin{minipage}{\textwidth}
    \centering
    \includegraphics[width=0.95\textwidth]{plots/figure1/legend.pdf}
\end{minipage}

\caption{Analogue of Figure~\ref{fig:main} for ViT architecture and the CIFAR-5m dataset for 3 layers of the network. For some of the figures we observe relatively larger gaps between $\text{Shampoo}^2$ and optimal Kronecker approximation.}
\label{fig:vit}
\end{figure}

In this subsection, we present the results for a Vision Transformer (ViT) architecture trained on the CIFAR-5m dataset. This architecture features a patch size of 4, a hidden dimension of 512, an MLP dimension of 512, 6 layers, and 8 attention heads.

For these experiments, we utilize three layers from the fourth transformer block: two layers from the MLP (referred to as 'FFN Linear Layer 1' and 'FFN Linear Layer 2') and the QK layer\footnote{The QK layer is separated from the V part of the layer, following similar decomposition method described by \citet{duvvuri2024combining}} (referred to as 'Q-K Projection Layer').

\section{Experiments} \label{app:exp}
\textbf{Datasets and Architectures.} We conducted experiments on three datasets: MNIST \citep{mnist}, CIFAR-5M \citep{cifar5m}, and ImageNet \citep{imagenet}, using logistic regression, ResNet18 \citep{resnet}, and ConvNeXt-T \citep{Liu_2022_CVPR} architectures, respectively. For MNIST, we subsampled two digits ($\{ 0, 1\}$) and trained a binary classifier.

\begin{table}[h]
  \caption{Summary of Experimental Configurations. $\lambda$ denotes weight decay and $\beta_1$ indicates momentum.}
  \label{experiment-table}
  \centering
    \small

\begin{tabular}{lllllllll}
    \toprule
    Dataset     & Architecture & Optimizer & Batch Size &  Steps & lr & $\lambda$ & $\beta_1$ \\
    \midrule
    MNIST       & Linear Classifier     & GD        & Full Batch & 25          & 0.01           & None           & 0       \\
    CIFAR-5M & ResNet18        & SGD        & 128         & 10000        & .02            & None             & .9      \\
    ImageNet    & ConvNeXt-T                  & AdamW        & 2048         & 50000       & 3e-3            & 5e-3             & 0.9      \\
    \bottomrule
\end{tabular}

\end{table}

For MNIST, we used the only layer, i.e, the first layer of the linear classifier for computing the cosine similarities. For Resnet18 and Imagenet, we picked arbitrary layers. In particular, for Resnet 18, we used one of the convolution layers within the first block ('layer1.1.conv1' in \url{https://pytorch.org/vision/master/_modules/torchvision/models/resnet.html#resnet18}). For Imagenet, we used the 1x1 convolutional layer within the 2nd block of convnext-T ('stages.2.1.pwconv1' in \url{https://pytorch.org/vision/main/models/generated/torchvision.models.convnext_tiny.html#torchvision.models.convnext_tiny}). 

\textbf{Cosine similarity estimation for $H_{\text{GN}}$.} For estimating the Frobenius norm of $H_{\text{GN}}$, we used the identity: 

\[ \E_{v \sim \mathcal{N}(0,I_d)} [v^\top H_{\text{GN}}^2 v] = \E_{v \sim \mathcal{N}(0,I_d)} [\| H_{\text{GN}} v \|_2^2] = \|H_{\text{GN}}\|_F^2 \]

Hessian-vector products with the Gauss--Newton component were performed using the DeepNetHessian library provided by \citet{papyan19}. 

For estimating the cosine similarity between $H_{\text{GN}}$ and its estimator $\widetilde{H}_{\text{GN}}$, we used the following procedure:
\begin{enumerate}
    \item Estimate $\| H_{\text{GN}} \|_F$, and calculate $\| \widetilde{H}_{\text{GN}} \|_F$.
    \item Define scaled $\widetilde{H}_{\text{GN}}$ as  $\widetilde{S}_{\text{GN}} = \frac{\| H_{\text{GN}} \|_F}{\| \widetilde{H}_{\text{GN}} \|_F} \widetilde{H}_{\text{GN}}$.
    \item $\text{Cos-sim}(H_{\text{GN}}, \widetilde{H}_{\text{GN}}) = 1 - \frac{\|H_{\text{GN}} - \widetilde{S}_{\text{GN}}\|_F^2}{2 \|H_{\text{GN}}\|_F^2 }$, where the numerator is again estimated via Hessian-vector products.
\end{enumerate}

Note that in the above procedure, we can exactly calculate $\| \widetilde{H}_{\text{GN}} \|_F$ as it is generally of a Kronecker product form with both terms of size $m \times m$ or $n \times n$, where $m \times n$ is the size of a weight matrix. 

\textbf{Cosine similarity estimation for $H_{\text{Ada}}$.} We follow a similar recipe as before, but using a difference method for computing the product $H_{\text{Ada}}v$. For a given time $T$, $H_{\text{Ada}} = \sum_{t=1}^T g_t g_t^\top$. Thus, $H_{\text{Ada}} v = \sum_{t=1}^T (g_t^\top v) g_t$. We maintain this by keeping a running estimate of the quantity for multiple random vectors $v$ during a training run, and use it for estimating the product $H_{\text{Ada}} v$.

\subsection{Figure details}
\label{app:fig_details}

\textit{Optimal Kronecker} method, wherever used was computed with five rounds of power iteration, starting from the identity. For $H = H_{\text{GN}}$, the Hessian approximations $\textit{Shampoo}^2$, \textit{Shampoo}, and \textit{K-FAC} were done using sampled labels and a batch size of $1$. For $H = H_{\text{Ada}}$ and step $t$, we used gradient enocoutered during the training run in steps $\leq t$.

\textit{K-FAC} was computed with the ``reduce'' variant from ~\cite{eschenhagen2023kroneckerfactored}.

In Figure~\ref{fig:whyI}, the \textit{Optimal Kronecker} legend represents the cosine similarity between the optimal Kronecker approximation of $H_{\text{GN}}$ and $H_{\text{GN}}$. This is precisely equal to $\frac{\sigma_1}{\sqrt{\sum_i \sigma_i^2}}$. Similarly, the label \textit{L} (resp. \textit{R}) represents the cosine similarity between the top left (resp. right) singular vector of $\hat{H}_{\text{GN}}$ and the estimate obtained after one round of power iteration starting from $I_n$ (resp. $I_m$). This is precisely equal to $\frac{\alpha_1 \sigma_1}{\sqrt{\sum_i \alpha_i^2 \sigma_i^2}}$.

In Figure~\ref{fig:figure4} (top), the Hessian approximation is calculated with batch size $1$, i.e, $|B|=1$ in Section~\ref{sec:labels}.
Similarly, in Figure~\ref{fig:figure4} (bottom), $|B| = 256$.

\section{Deferred proofs} \label{app:proofs}
\lemnopsd*
\begin{proof}
    Consider two PSD matrices $M_1$ and $M_2$ having the eigenvalue decomposition $M_1 = \sum \lambda_{1i} q_{1i}q_{1i}^\top $ and $M_2 = \sum \lambda_{2i} q_{2i}q_{2i}^\top$. Then 

    \[ \text{Tr}(M_1M_2) = \sum_{i,j} \lambda_{1i} \lambda_{2j} \left(q_{1i}^\top q_{2j}\right)^2 \]

    Thus, if $M_1$ and $M_2$ have unit frobenius norm and $M_1$ is positive definite, then $\text{Tr}(M_1M_2) > 0$. 

    Thus, if $V_1$ is positive definite, then by orthogonality of successive singular vectors, $V_i$ for $i \geq 2$ cannot be positive semi-definite.
\end{proof}

\lemiden*  

\begin{proof}
    Consider the eigendecomposition of any $M \in S_q$ given by $\sum_{i=1}^q \lambda_i v_iv_i^\top$. Denote $L  = \{i: \lambda_i \leq \frac{1}{\sqrt{q}} \}$. As $\sum \lambda_i^2 = 1$, therefore, $|A| \geq 1$. Consider any $j \in A$. Then

    \[ \langle Vec(M), Vec(v_j v_j^\top) \rangle \leq \frac{1}{\sqrt{q}} \]

    As $v_j$ is orthogonal to the other eigenvectors. Thus, we can see

    \[ \max_{M \in S_q} \min_{M' \in S_q} \langle \vect{M}, \vect{M'} \rangle \leq \frac{1}{\sqrt{q}} \]

    Moreover, for the matrix $\frac{1}{\sqrt{q}} I_q$, for any matrix $M'$,

    \[ \frac{1}{\sqrt{q}} \langle I_q, M' \rangle = \frac{\text{tr}(M')}{\sqrt{q}} \]

    where $\text{tr}(M')$ denotes the trace of the matrix $M'$. However, we know $\text{tr}(M') = \sum \lambda_i \geq 1$ as $\sum \lambda_i^2 = 1$. Thus

    \[ \frac{1}{\sqrt{q}} \langle I_q, M' \rangle = \frac{\text{tr}(M')}{\sqrt{q}} \geq \frac{1}{\sqrt{q}} \]

    Note that this is the only matrix with this property as any other matrix will at least have one eigenvalue less than $\frac{1}{\sqrt{q}}$. Thus

    \[ \frac{1}{\sqrt{q}} I_q = \argmax_{M \in S_q} \min_{M' \in S_q} \langle \vect{M}, \vect{M'} \rangle \]
\end{proof}

\lemavggrad*
\begin{proof}
    Evaluating $G_{B,\bf{s}}  G_{B,\bf{s}}^T$, we get
    \[ G_{B,\bf{s}}  G_{B,\bf{s}}^T = \frac{1}{|B|^2} \sum_{\substack{x,x' \in B, \\ s=\textbf{s}[x], s'=\textbf{s}[x']}} G_{x,s} G_{x',s'}^\top \]
    Taking the expectation over $\textbf{s}$ for a given $B$, and by using $\E_s[G_{x, s}] = 0$ we get
    \[ \E_{\textbf{s}} [G_{B,\bf{s}}  G_{B,\bf{s}}^T] = \frac{1}{|B|^2} \sum_x \E_{s \sim f(x)} [G_{x,s} G_{x,s}^\top] = \frac{1}{|B|} \E_{x \sim B, s \sim f(x)}[G_{x,s}G_{x,s}^\top] \]
    Now taking an expectation over batches, we get
    \[|B| \E_{B,\bf{s}} [G_{B,\bf{s}}  G_{B,\bf{s}}^T] = \E_{x, s \sim f(x)}[ G_{x, s} G_{x, s}^T ]\]
\end{proof}

\lemrealbatch*
\begin{proof}
    Evaluating $G_{B}  G_{B}^T$, we get
    \[ G_{B}  G_{B}^T = \frac{1}{|B|^2} \sum_{(x,y), (x',y') \in B} G_{x,y} G_{x',y'}^\top \]
    Taking the expectation over $B$ on both the sides, we get
    \begin{align*}
    &\E_B \left[G_{B}  G_{B}^T\right] = \frac{1}{|B|^2} \left[ |B| \E_{x,y}[G_{x,y}G_{x,y}^\top] + (|B|^2-|B|) \E_{x,y}[G_{x,y}] \E_{x,y}[G_{x,y}]^\top \right] \\
    \implies &\E_B \left[G_{B}  G_{B}^T\right] = \frac{1}{|B|}\E_{x,y}[G_{x,y}G_{x,y}^\top] + \left(1 - \frac{1}{|B|}\right) \E_{x,y}[G_{x,y}] \E_{x,y}[G_{x,y}]^\top \\
    \end{align*}
\end{proof}

\section{Technical Background on Hessian}
\label{app:tech_hess}

\textbf{Gauss--Newton (GN) component of the Hessian.}
For a datapoint $(x,y)$, let $f(x)$ denote the output of a neural network and $\mathcal{L}(f(x),y)$ represent the training loss. Let $W \in \mathbb{R}^{m \times n}$ represent a weight matrix in the neural network and $\mathcal{D}$ denote the training distribution. Then, the Hessian of the loss with respect to $W$ is given by

\[ \E_{(x,y) \sim \mathcal{D}} \left[ \frac{\partial^2 \mathcal{L}}{\partial W^2} \right] = \E_{(x,y) \sim \mathcal{D}} \left[\frac{\partial f}{\partial W} \frac{\partial^2 \mathcal{L}}{\partial f^2} \frac{\partial f}{\partial W}^\top\right] + \E_{(x,y) \sim \mathcal{D}}\left[\frac{\partial \mathcal{L}}{\partial f} \frac{\partial^2 f}{\partial W^2}\right]. \]

The first component, for standard losses like cross-entropy (CE) and mean squared error (MSE), is positive semi-definite and is generally known as the Gauss--Newton (GN) component ($H_{\text{GN}}$). Previous works have shown that this part closely tracks the overall Hessian during neural network training \citep{Sankar_Khasbage_Vigneswaran_Balasubramanian_2021}, and thus most second-order methods approximate the GN component. Denoting $\frac{\partial \mathcal{L}(f(x),y)}{\partial W}$ by $G_{x,y} \in \mathbb{R}^{m \times n}$ and $g_{x,y} = \vect{G_{x,y}}$, for CE loss, it can also be shown that

\[ H_{\text{GN}} = \E_{(x,y) \sim \mathcal{D}}\left[\frac{\partial f}{\partial W} \frac{\partial^2 \mathcal{L}}{\partial f^2} \frac{\partial f}{\partial W}^\top\right] = \E_{\substack{x \sim \mathcal{D}_x \\ s \sim f(x)}} \left[g_{x,s} g_{x,s}^\top \right] , \]

\section{Related work} \label{app:rel_work}
The literature related to second order optimization within deep learning is very rich, with methods that can be broadly classified as Hessian-free and methods based on estimating the preconditioner $H$ (which could refer to either $H_{\text{Ada}}$ or $H_{\text{GN}}$). Hessian-free methods \citep{martens10} generally tend to approximate the preconditioned step (for Newton's method) using Hessian vector products, but do not maintain an explicit form of the Hessian. Estimating $H$ \citep{martens15, gupta18} methods maintain an explicit form of the preconditioner that could be efficiently stored as well as estimated. 

\subsection{Hessian-free} 
One of the seminal works related to second order optimization within deep learning was the introduction of Hessian-free optimization \citep{martens10}. The work demonstrated the effectiveness of using conjugate gradient (CG) for approximately solving the Newton step on multiple auto-encoder and classifications tasks. Multiple works \citep{martens11, Minhyung15} have extended this algorithm to other architectures such as recurrent networks and multidimensional neural nets. One of the recent works \citep{garcia2023fisherlegendre} also takes motivation from this line of work, by approximately using single step CG for every update, along with maintaining a closed form for the inverse of the Hessian, for the single step to be effective.

\subsection{Estimating Preconditioner}
Given that it is costly to store the entire matrix $H$, various works have tried to estimate layer-wise $H$. KFAC \citep{martens15} was one of the first work, that went beyond diagonal approximation and made a Kronecker product approximation  to layer-wise $H_{\text{GN}}$. It showed that this structure approximately captures the per layer Hessian for MLPs. This approximation was extended to convolutional \citep{kazuki19} and recurrent \citep{martens2018kroneckerrecurrent} architectures. Subsequent works also improved the Hessian approximation, by further fixing the trace \citep{Gao_Liu_Huang_Wang_Wang_Xu_Yu_2021} as well as the diagonal estimates \citep{Thomas18, gao2020eigenvaluecorrected} of the approximation. A recent work \citep{eschenhagen2023kroneckerfactored} also demonstrated that K-FAC can be extended to large-scale training. 

From the viewpoint of approximating Adagrad \citep{duchi11}, \citet{gupta18} introduced Shampoo, that also makes a Kronecker product approximation to $H_{\text{Ada}}$. One of the subsequent work \citep{yi21} introduced a modification of Shampoo, that was precisely estimating the layer-wise $H_{\text{GN}}$ under certain distributional assumptions. Other works \citep{anil2021scalable} introduced a distributed implementation of Shampoo, that has recently shown impressive performance for training large scale networks \citep{shi2023distributed}. Recently, another paper \citep{duvvuri2024combining} proposed a modification of Shampoo, empirically and theoretically demonstrating that the new estimator approximates $H_{\text{Ada}}$ better than Shampoo's approximation. Our work shows that the square of Shampoo's approximation of $H_{\text{Ada}}$ is nearly equivalent to the optimal Kronecker approximation.

\section{Comparison with extra square root in Adagrad based approaches} \label{app:sq_root}

Multiple previous works \citep{balles2020geometry, lin2024remove} have tried to address the question of why Adagrad-based approaches like Adam and Shampoo, have an extra square root in their update compared to Hessian inverse in their updates. This question is primarily concerned with the final update to the weights being used in the optimization procedure, once we have approximated the Hessian.

The primary contribution of this work is completely orthogonal to this question. We are addressing the question of optimal Kronecker approximation of the Hessian, and its connection to Shampoo's Hessian approximation. This is orthogonal to the Hessian power used in the final update.

\end{document}